\documentclass[letterpaper, 10 pt, conference]{ieeeconf}  % Comment this line out
\IEEEoverridecommandlockouts                              % This command is only
\title{\LARGE \bf
Learning from Local Walks on Dynamic Graphs with Bandit Feedback
}

\author{Sourav Chakraborty$^*$, Amit Kiran Rege$^*$, Claire Monteleoni, Lijun Chen% <-this % stops a space
\thanks{$^*$Equal Contribution}% <-this % stops a space
\thanks{All authors are with the Department of Computer Science,
        University of Colorado, Boulder, CO 80309, USA
        {\tt\footnotesize \{sourav.chakraborty, amit.rege, cmontel, lijun.chen\}@colorado.edu}}%
\thanks{Claire Monteleoni is also with INRIA, Paris, France.}
}

 \usepackage[
backend=biber,
style=numeric-verb,
 ]{biblatex}
 \usepackage[
colorlinks=true,
  citecolor=red,
  urlcolor=black,
  linkcolor=blue
]{hyperref} 

\bibliography{refs}
\usepackage{subcaption}
\usepackage{svg}
\usepackage{caption}
\let\labelindent\relax
\usepackage{url}            % simple URL typesetting
\usepackage{booktabs}       % professional-quality tables
\usepackage{amsfonts}       % blackboard math symbols
\usepackage{nicefrac}       % compact symbols for 1/2, etc.
\usepackage{microtype}      % microtypography
\usepackage{xcolor}         % colors

\usepackage{algorithm}

\usepackage{amssymb}		% to get all AMS symbols
\usepackage{graphicx}		% to insert figures
\usepackage{amsmath}
\usepackage{amsthm}

\usepackage{algorithm}
\usepackage{algpseudocode}
\usepackage{amsmath}

\usepackage{comment}
\usepackage{multirow}
\usepackage{enumitem}
\usepackage{subcaption}
\usepackage{tikz}
\usepackage[dvipsnames,table,xcdraw]{xcolor}

\newtheorem{theorem}{Theorem}
\newtheorem*{theorem*}{Theorem}

\newtheorem{lemma}{Lemma}
\newtheorem*{lemma*}{Lemma}

\newtheorem{definition}{Definition}
\newtheorem*{definition*}{Definition}

\newtheorem{corollary}{Corollary}
\newtheorem*{corollary*}{Corollary}

\newtheorem*{example*}{Example}

\newtheorem{proposition}{Proposition}
\newtheorem*{proposition*}{Proposition}

\newtheorem*{claim*}{Claim}

\newtheorem*{remark*}{Remark}

\newtheorem*{conjecture*}{Conjecture}

\newtheorem*{assumption*}{Assumption}

\usepackage{amsmath}  % for \frac
\usepackage{xspace}   % for smart spacing

\usepackage{pifont}
\usepackage{float}

\usepackage{tikz}
\usetikzlibrary{calc,positioning,fadings}

\definecolor{darkblue}{RGB}{0, 30, 90}
\definecolor{framegrey}{RGB}{200, 200, 200}
\definecolor{fadeblue}{RGB}{230, 240, 255}
\definecolor{graphgrey}{RGB}{140, 140, 140}

\usepackage{amsmath,amssymb,amsthm,mathtools}

\usepackage[nameinlink]{cleveref}  % nameinlink ensures entire "Theorem 1.1" is linked
\crefname{theorem}{Theorem}{Theorems}
\Crefname{theorem}{Theorem}{Theorems}

\crefname{lemma}{Lemma}{Lemmas}
\Crefname{lemma}{Lemma}{Lemmas}

\crefname{proposition}{Proposition}{Propositions}
\Crefname{proposition}{Proposition}{Propositions}

\crefname{assumption}{Assumption}{Assumptions}
\Crefname{assumption}{Assumption}{Assumptions}

\crefname{algorithm}{Algorithm}{Algorithms}
\Crefname{algorithm}{Algorithm}{Algorithms}

\crefname{figure}{Figure}{Figures}
\Crefname{figure}{Figure}{Figures}

\crefname{corollary}{Corollary}{Corollaries}
\Crefname{corollary}{Corollary}{Corollaries}

\crefname{remark}{Remark}{Remarks}
\Crefname{remark}{Remark}{Remarks}

\crefname{definition}{Definition}{Definitions}
\Crefname{definition}{Definition}{Definitions}

\crefname{section}{Section}{Sections}
\Crefname{section}{Section}{Sections}

\crefname{subsection}{Section}{Sections}
\Crefname{subsection}{Section}{Sections}

\defbibheading{centered}[\refname]{%
  \centering\section*{#1}%
  \addcontentsline{toc}{section}{#1}%
  \markboth{#1}{#1}%
}

\begin{document}

\maketitle
\thispagestyle{empty}
\pagestyle{empty}

\begin{refsection}

% \begin{abstract}
% We study stochastic multi-armed bandits on dynamic graphs, where a learner can play \sourav{Need better narrative flow. haven't defined the nodes/arm equivalence etc.} only its current node or a current neighbor. In this setting, finding the best arm is not enough: after exploration ends, the learner may still need many rounds to reach it through a changing graph. We therefore ask when the graph's own local walk is enough for both exploration and post-exploration navigation, and we give a condition that keeps this walk stable over time and rules out long stretches of poor connectivity. Under this condition, we analyze three explore-then-commit algorithms and prove sublinear expected regret, including a reward-aware variant with a separate gain theorem.
% \end{abstract}

\begin{abstract}
We study stochastic multi-armed bandits on dynamic graphs, where arms correspond to the vertices of a network with time-varying edges. In this setting, the learner is restricted to local movement, selecting only its current node or an immediate neighbor at each round. This constraint decouples best-arm identification from exploitation: even after the optimal arm is identified, the learner may remain unable to reach it through the evolving topology. We identify a process-agnostic structural condition, based on sliding-window mixing, that ensures the graph's intrinsic walk remains stable for both exploration and navigation. Under this regime, we analyze a family of local explore-then-commit algorithms and establish sublinear expected regret. Our framework includes a reward-aware strategy, for which we prove a worst-case safety theorem and a separate performance gain theorem.
\end{abstract}

\section{Introduction}
\label{sec:introduction}

% In the stochastic multi-armed bandit problem, a learner repeatedly chooses among several actions, observes only the reward of the chosen action \sourav{might need to show arm equivalence as well for future.}, and seeks to keep its cumulative regret small \sourav{not defined regret, so we can stick with maximize reward here and then later we can show the equivalence.}\cite{auer2002finite,lattimore-szepesvari2020,slivkins2019introduction}. Classical bandit models assume that every arm can be chosen at every round. This is too strong in many networked settings. A mobile collector in a sensor network, a search routine on a peer-to-peer overlay, or a patrol agent moving through a changing environment cannot jump to an arbitrary location in one step \cite{lima-barros2007sensorwalk,gkantsidis-mihail-saberi2004p2p,basilico2022patrolling} \sourav{its slightly ambiguous here. need better explainations of the examples}. Its next action is limited by where it currently is and by which edges are presently available.

In the stochastic multi-armed bandit (MAB) problem, a learner repeatedly chooses among several actions (interchangeably referred to as \textit{arms}) observes only the reward of the selected action, and seeks to maximize its cumulative reward (\cite{auer2002finite,lattimore-szepesvari2020,slivkins2019introduction}). Classical bandit models typically assume that every arm is available for selection at every round. This assumption, however, is too strong for many networked settings where movement is restricted and inherently local. For instance, a mobile data collector in a sensor network is constrained by physical proximity, requiring it to be near a specific node to harvest its data \cite{lima-barros2007sensorwalk}; a search routine on a peer-to-peer overlay is confined to traversing logical links between neighboring peers \cite{gkantsidis-mihail-saberi2004p2p}; and a patrol agent is restricted to navigable pathways that may open or close as the environment changes \cite{basilico2022patrolling}. In these scenarios, the learner cannot jump to an arbitrary location; instead, its next action is strictly determined by its current position and the connectivity available at that moment.

This motivates the dynamic graph bandit problem, where arms correspond to vertices in a network with time-varying edges. The learner observes only its local neighborhood, moves along locally available edges, and receives feedback solely from visited nodes. Consequently, the evolving topology governs both immediate reward availability and future reachability. This fundamentally differs from feedback-graph bandits, where graphs encode side-observations rather than movement constraints (\cite{mannor-shamir2011,alon2017feedback}), and from static-graph bandits, where fixed, globally known topologies permit pre-planned navigation (\cite{zhang-johansson-li2023graphbandit,paschalidis-zhang-li2024multigraphbandit}).

A core challenge in this setting is that best-arm identification is decoupled from exploitation. In classical bandits, identifying the optimal arm immediately enables perpetual exploitation. Here, a learner might identify the best arm but remain temporarily unable to reach it. Learnability thus requires the environment to consistently provide navigable pathways; otherwise, the learner could be structurally isolated indefinitely. Importantly, standard full-horizon
  conditions, such as the connectivity of the union of all graphs, are insufficient, as a network might be connected in aggregate but disconnected at almost every individual step. A valid structural condition must therefore be
  \textit{shift-invariant}, holding from any arbitrary round. This guarantees navigability during the post-exploration phase, regardless of the random stopping time at which exploration concludes. We define a condition that meets these criteria (see Section  \ref{subsec:walk-centric-condition}), avoiding the pitfalls of whole-horizon summaries.

To isolate the distinct challenges of information gathering and subsequent movement, we adopt the explore-then-commit framework. Our analysis centers on the canonical walk (see \eqref{eq:walk_matrix}), where the learner either remains at its current node or moves uniformly to an available neighbor. We prioritize this simple local kernel because it is entirely decentralized and allows us to address a fundamental question:\textit{ when is a graph’s intrinsic walk sufficient for learning?} While more complex kernels can be useful, the canonical walk serves as a transparent analytical probe that exposes the structural constraints of the environment before asking whether reward-biased movement yields further gains. This choice makes the dichotomy between statistical identification and navigation explicit in both algorithms and proofs.

% Our answer is based on a condition on the above walk. Roughly speaking, the walk should have the same long-run behavior at every round (mixing), and every short time window should contain enough rounds that do not trap the walk in one small part of the graph. It also matches the positive regime identified in the dynamic random-walk literature, where walks with a common stationary distribution behave much more like walks on a fixed graph \cite{avin-koucky-lotker2008,olshevsky-tsitsiklis2013,sauerwald-zanetti2019,shimizu-shiraga2022}. The condition naturally covers several fixed-degree dynamic graph families, such as regular overlay schedules and degree-preserving rewiring.

To formalize shift-invariant navigability, we impose a condition directly on the canonical walk. Intuitively, we require the walk to maintain a consistent long-run stationary distribution, and we require that every short time window contains enough ``well-connected" rounds to prevent the learner from becoming trapped in a local subgraph. We formalize this as a \textit{common-stationary sliding-window mixing} property. By ensuring a stable stationary distribution, we align with positive regimes in the dynamic random-walk literature, where such walks mimic their static counterparts (\cite{avin-koucky-lotker2008,olshevsky-tsitsiklis2013,sauerwald-zanetti2019,shimizu-shiraga2022}). This encompasses natural fixed-degree processes, such as rotating regular overlays and degree-preserving rewirings.

Within this regime, we analyze a family of local explore-then-commit algorithms that utilize walk-based exploration with both fixed and adaptive stopping rules (see Section \ref{sec:problem_and_principle}). Our analysis extends to reward-aware movement, where we formally decouple worst-case safety from provable performance gains. Unlike model-specific analyses \cite{fmab2026}, we establish process-agnostic structural conditions (see Section \ref{sec:algorithms_and_analysis}) that ensure learnability for any such walk-based approach. 
Due to space constraints, extended related work, technical derivations, and complete proofs are deferred to the appendix.

% Due to space constraints, extended related work, technical derivations, and complete proofs are deferred to the full version \cite{anonymous_cdc_graph_2026}.

%\input{Related Work/r1}

\section{Model and Problem Formulation}
\label{sec:problem_and_principle}

This section defines the learning problem, fixes the information pattern, and introduces the structural condition used throughout the paper.

\subsection{Bandit learning on dynamic graphs}
\label{subsec:dynamic-graph-bandits}

Let $A=\{1,\dots,n\}$ be a finite set of arms, where each $a\in A$ is associated with a reward distribution $\mathcal D(a)$ supported on $[0,1]$ with mean $\mu(a)$. Rewards are i.i.d. across time and independent of the graph process. We assume a unique optimal arm $a^\star \triangleq \arg\max_{a\in A}\mu(a)$ and define the suboptimality gaps as $\Delta(a)\triangleq\mu(a^\star)-\mu(a)$ for $a\neq a^\star$. At each discrete round $t\ge 1$, the environment presents the learner an undirected graph $G_t=(A,E_t)$. The learner begins at node $a_{t-1}$ and observes only the local neighborhood $L_t(a_{t-1}) \triangleq N_t(a_{t-1})\cup\{a_{t-1}\}$, where $N_t(u)\triangleq\{v\in A:(u,v)\in E_t\}$. The graph process is oblivious: the sequence $\{G_t\}_{t\ge 1}$ is independent of the learner's past actions and randomization. The learner chooses $a_t\in L_t(a_{t-1})$, moves to that node, and receives reward $r_t \sim \mathcal D(a_t)$. We measure performance by the regret $R(T)\triangleq\sum_{t=1}^T (\mu(a^\star)-r_t)$, whose expectation matches the standard pseudo-regret $\mathbb E[R(T)] = \sum_{t=1}^T \mathbb E[\mu(a^\star)-\mu(a_t)]$. Our objective is to design a local policy that is learnable, meaning it achieves sublinear expected regret: $\lim_{T\to\infty} \mathbb E[R(T)]/T = 0$. While the full graphs $\{G_t\}$ are used for analysis, the policy remains strictly local, seeing only the current neighborhood $L_t(a_{t-1})$.

\subsection{Why the condition must be stated on the walk}
\label{subsec:walk-centric-condition}

\begin{comment}
Local movement makes learnability sensitive to the temporal arrangement of edges. A connected union over the whole horizon is not enough: the learner can still be trapped away from $a^\star$ for a linear number of rounds. Appendix~\ref{app:weaker-conditions-fail} gives a simple construction. A second obstruction is that a whole-horizon count of good blocks is also too weak. Even if a fixed fraction of blocks are individually well connected, all of those blocks may occur at the very end of the horizon. In that case the learner again suffers linear regret before the graph ever becomes useful. Appendix~\ref{app:weaker-conditions-fail} gives this construction as well.

These counterexamples point to two requirements. The first is that the condition must control the actual time-inhomogeneous random walk, not only static union graphs. The second is that the condition must hold uniformly over time, so that it remains available after an arbitrary stopping time. This is particularly important for the navigation phase, which begins after exploration stops at a random time. 
\end{comment}

Local movement constraints make learnability highly sensitive to the temporal arrangement of edges. As discussed, connectivity of the horizon-wide union graph (all edges ever active) is insufficient, as a learner can be structurally isolated from $a^\star$ for a linear number of rounds despite global connectivity (see Appendix \ref{app:horizon-long}). Similarly, a high density of time windows (blocks) with good connectivity across the full horizon is inadequate if those blocks appear only at the end (see Appendix \ref{app:global-fraction-insufficient}). These failures necessitate two requirements: (i) the condition must directly control the induced time-inhomogeneous random walk rather than static graph summaries, and (ii) it must hold uniformly over time to remain valid after an
arbitrary stopping time. This shift-invariance is essential for the navigation phase, which initiates at a random time when exploration terminates.
%\sourav{not clear how those counterexamples lead to the abovementioned requirements. It could have led to having better conditions on the graph sequence as well.}

% The canonical walk associated with the current graph $G_t$ is
% \begin{equation}
% \label{eq:walk_matrix}
% U_t(u,v)
% =
% \begin{cases}
% \dfrac{1}{1+\deg_t(u)} & \text{if } v\in N_t(u)\cup\{u\},\\[6pt]
% 0 & \text{otherwise,}
% \end{cases}
% \end{equation}
% where $\deg_t(u)\triangleq|N_t(u)|$. This is the exact kernel used by the baseline algorithm studied below.

% For a reversible Markov kernel $P$ with stationary distribution $\pi$, let
% $1=\lambda_1(P)\ge \lambda_2(P)\ge \cdots \ge \lambda_n(P)\ge -1$
% denote its eigenvalues in nonincreasing order. We write $\operatorname{gap}_{\mathrm{abs}}(P)
% \triangleq
% 1-\max\{|\lambda_2(P)|,|\lambda_n(P)|\}$
% for its absolute spectral gap. For reversible chains this is the correct one-step contraction parameter on the subspace of $\pi$-mean-zero functions.

The canonical walk for graph $G_t$ is defined by the kernel:
\begin{equation}
\label{eq:walk_matrix}
U_t(u,v) = \begin{cases} (1+\deg_t(u))^{-1} & \text{if } v \in N_t(u) \cup \{u\}, \\ 0 & \text{otherwise,} \end{cases}
\end{equation}

where $\deg_t(u)\triangleq|N_t(u)|$. A large absolute spectral gap means the walk converges quickly to its stationary distribution $\pi$ from any starting point — informally, it means the learner cannot get stuck in
   one region of the graph for long.
   Since $G_t$ is undirected, $U_t$ is reversible with stationary distribution $\pi$ determined by the degree sequence: each arm $a$ is visited with long-run frequency $\pi(a)$. For a reversible kernel $P$ with eigenvalues $\lambda_1(P)\ge\cdots\ge\lambda_n(P)$, the absolute spectral gap $\operatorname{gap}_{\mathrm{abs}}(P)\triangleq 1-\max{|\lambda_2(P)|,|\lambda_n(P)|}$ controls how quickly the walk approaches this visit distribution from any starting node — a larger gap guarantees faster coverage of all arms. 

\begin{figure}[t]
    \centering
    \resizebox{\columnwidth}{!}{
        \begin{tikzpicture}

            % --- Main Diagram Title ---
            \node[font=\bfseries] (title) at (0, 3.9) {Time Window $W=5$};

            % --- Circles and Content Logic ---
            \foreach \i in {1, ..., 5} {
                % Spacing at 1.8cm is the "sweet spot" for a single column
                \pgfmathsetmacro{\xpos}{(\i - 3) * 1.8}
                
                \begin{scope}[xshift=\xpos cm, yshift=2.5cm] 
                    
                    % Draw the frame circle
                    \draw[framegrey, thin] (0,0) circle (0.8cm);
                    \node[anchor=north] at (0, -0.9) {$t_\i$};
                    
                    % Define node coordinates
                    \foreach \n in {1,...,12}{ \coordinate (outer\n) at (\n*30:0.75cm); }
                    \foreach \n in {1,...,8}{ \coordinate (middle\n) at (\n*45:0.45cm); }
                    \foreach \n in {1,...,4}{ \coordinate (inner\n) at (\n*90:0.25cm); }
                    \coordinate (center) at (0,0);

                    % Mixing Logic (T1, T3 are blue)
                    \def\isdense{0}
                    \ifnum\i=1 \def\isdense{1} \fi
                    \ifnum\i=3 \def\isdense{1} \fi

                    \ifnum\isdense=1
                        % ==========================================
                        % DENSE/MIXING ROUNDS (T1, T3)
                        % ==========================================
                        \begin{scope}
                            \clip (0,0) circle (0.8cm); 
                            \node[circle, inner sep=0pt, minimum size=1.6cm, fill=fadeblue, path fading=west, opacity=0.8] at (0,0) {};
                            
                            % Draw thick blue edges for visibility
                            \draw[darkblue, thin] (inner1) -- (outer1) (inner1) -- (outer3);
                            \draw[darkblue, thin] (inner2) -- (outer4) (inner2) -- (outer6);
                            \foreach \n in {1,...,8}{\draw[darkblue, thin] (middle\n) -- (center);}
                            \foreach \n in {1,...,12}{
                                \pgfmathtruncatemacro{\nextn}{mod(\n,12)+1}
                                \draw[darkblue, thin] (outer\n) -- (outer\nextn);
                            }
                            
                            % BOLDER NODES: Increased radius for visibility
                            \foreach \n in {1,...,12}{\draw[fill=darkblue, darkblue] (outer\n) circle (0.05cm);}
                            \foreach \n in {1,...,8}{\draw[fill=darkblue, darkblue] (middle\n) circle (0.05cm);}
                            \draw[fill=darkblue, darkblue] (center) circle (0.03cm);
                        \end{scope}
                    \else
                        % ==========================================
                        % SPARSE ROUNDS (T2, T4, T5)
                        % ==========================================
                        \begin{scope}
                            \clip (0,0) circle (0.8cm); 
                            
                            % Draw a few sparse edges
                            \ifnum\i=2 \draw[graphgrey, thin] (outer1) -- (outer4) (middle2) -- (outer3); \fi
                            \ifnum\i=4 \draw[graphgrey, thin] (outer2) -- (outer5) (inner4) -- (center); \fi
                            \ifnum\i=5 \draw[graphgrey, thin] (outer10) -- (outer1) (inner3) -- (center); \fi

                            % BOLDER GREY NODES: Increased radius and solid color
                            \foreach \n in {1,...,12}{\draw[fill=graphgrey, graphgrey] (outer\n) circle (0.04cm);}
                            \foreach \n in {1,...,8}{\draw[fill=graphgrey, graphgrey] (middle\n) circle (0.03cm);}
                            \draw[fill=graphgrey, graphgrey] (center) circle (0.02cm);
                        \end{scope}
                    \fi
                \end{scope}
            }
        \end{tikzpicture}
    }
    \caption{Common-stationary sliding-window mixing ($W=5$, $\rho=0.4$). Mixing rounds (blue) provide recurring spectral gap contraction, ensuring learnability even after random stopping times.}
    \label{fig:mixing_window}
\end{figure}
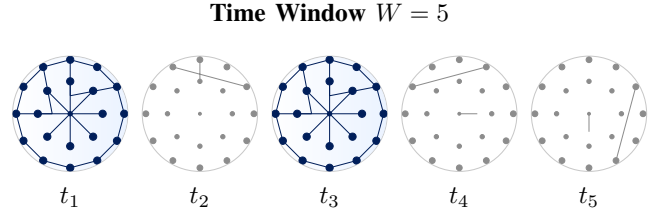

% \begin{definition}[Common-stationary sliding-window mixing]
% \label{def:mixing-sequence}
% Fix $W\in\mathbb N$, $\rho\in(0,1]$, $\gamma\in(0,1]$, and a probability distribution $\pi$ on $A$. An infinite graph sequence $\{G_t\}_{t\ge 1}$ is called $(W,\rho,\gamma,\pi)$-mixing for the canonical walk if the following two properties hold.

% First, every kernel $U_t$ in \eqref{eq:walk_matrix} is reversible with stationary distribution $\pi$.

% Second, every interval of $W$ consecutive times contains at least $\lceil \rho W\rceil$ indices $t$ such that
% \[
% \operatorname{gap}_{\mathrm{abs}}(U_t)\ge \gamma.
% \]
% \end{definition}

\begin{definition}[Common-stationary sliding-window mixing]
\label{def:mixing-sequence}
Fix $W\in\mathbb N$, $\rho, \gamma\in(0,1]$, and a distribution $\pi$ on $A$. A graph sequence $\{G_t\}_{t\ge 1}$ is $(W,\rho,\gamma,\pi)$-mixing for the canonical walk if: (i) each kernel $U_t$ in \eqref{eq:walk_matrix} is reversible with stationary distribution $\pi$, and (ii) every interval of $W$ consecutive rounds contains at least $\lceil \rho W\rceil$ indices $t$ such that $\operatorname{gap}_{\mathrm{abs}}(U_t)\ge \gamma$.
\end{definition}

Intuitively, this condition ensures that within any time window of length $W$, there is at least a fraction $\rho$ of ``well-connected'' rounds (where the spectral gap is at least $\gamma$). This guarantees the walk continuously mixes
toward a stable stationary distribution $\pi$ (see Figure \ref{fig:mixing_window}). For the canonical walk, a common stationary distribution arises naturally whenever all snapshots share the same degree sequence; in particular, $d$-regular snapshots yield a uniform $\pi$ (formalized in Appendix \ref{app:mixing-condition}). An oblivious graph process is $(W,\rho,\gamma,\pi)$-mixing if its realized sequence satisfies Definition \ref{def:mixing-sequence} almost surely. Rather than a universal requirement for all dynamic graphs, this identifies a stable regime where the natural local walk is robust enough to support a process-agnostic learnability theorem.

\subsection{Scope and model instantiations}
\label{subsec:scope-instantiations}

% Definition~\ref{def:mixing-sequence} can be understood as a characterization of when the canonical walk itself is a sound exploration primitive. In dynamic random-walk theory one can obtain broader positive results by changing the walk, for example by using Metropolis-type kernels that enforce a time-homogeneous stationary distribution by design~\cite{shimizu-shiraga2022}. That is a different question. The present paper asks when the simplest local rule already attached to the graph, namely the canonical walk in \eqref{eq:walk_matrix}, is enough for sublinear regret. We study this regime first because it isolates the structural content of the problem before introducing more engineered kernels.
Definition \ref{def:mixing-sequence} characterizes when the canonical walk serves as a robust exploration primitive. While dynamic random-walk theory offers broader results via engineered Metropolis-type kernels that enforce stationarity by design \cite{shimizu-shiraga2022}, we prioritize the simplest local rule embedded in the graph: the canonical walk \eqref{eq:walk_matrix}. This focus isolates the fundamental structural constraints of the problem before considering more complex, task-specific kernels.

% Furthermore, a whole-horizon count of good blocks is not preserved under truncation and says nothing about the random suffix that remains after exploration stops. Our condition is formulated directly on the walk and is shift-invariant by construction. The same assumption therefore remains available after a stopping time, which is exactly what the navigation phase needs.
Furthermore, unlike whole-horizon block counts, which are not preserved under truncation and fail to characterize the random suffix following exploration, our walk-centric condition is shift-invariant by construction. This ensures that the necessary structural assumptions remain available throughout the navigation phase, regardless of the random stopping time generated by the learner’s exploration.

% Several natural classes of graph processes fit this regime. The simplest is the regular case, in which every snapshot is $d$-regular and the common stationary distribution is uniform. This covers time-varying regular communication overlays, rotating regular patrol schedules, and periodic sequences of bounded-degree expander snapshots. A second class is fixed-degree rewiring, where all snapshots share the same degree sequence but not necessarily the same edges. Degree-preserving switch dynamics and related rewiring models generate this type of evolution and are standard tools for sampling or randomizing graphs with prescribed degrees~\cite{feder2006localswitch,stauffer-barbosa2005,fosdick2018configuring}. A third class is deterministic or stochastic schedules built from a finite family of fixed-degree graphs, such as alternating matchings or overlay rotations, provided every sliding window contains a positive density of snapshots with a uniform absolute spectral gap. Appendix~\ref{app:mixing-condition} makes the common-stationary implication of a fixed degree sequence precise.
Several natural graph processes fit this regime. The simplest is the regular case, where each snapshot is $d$-regular and $\pi$ is uniform, covering time-varying communication overlays, rotating patrol schedules, and periodic expanders. A second class is fixed-degree rewiring, where snapshots share a degree sequence but differ in edge topology. This includes degree-preserving switch dynamics, which are standard for sampling graphs with prescribed degrees (\cite{feder2006localswitch,stauffer-barbosa2005,fosdick2018configuring}). A third class consists of deterministic or stochastic schedules built from finite families of fixed-degree graphs, such as alternating matchings or overlay rotations, provided every sliding window maintains a positive density of snapshots with a uniform spectral gap. Appendix \ref{app:mixing-condition} details how fixed degree sequences imply a common stationary distribution.

Recent process-specific works analyze models such as i.i.d.\ Erdős–Rényi or edge-Markovian processes (where the stationary distribution of walk-kernel shifts with realized degrees) by exploiting their particular distributional structure (\cite{fmab2026}). This paper is complementary to that line; it trades model specificity for a process-agnostic structural theorem in the regime where the canonical walk remains analyzable. This scope determines the performance benchmark. The following explore-then-commit procedures prove structural learnability by decomposing the problem into coverage, estimation, and navigation under dynamic local movement. They are not presented as gap-optimal regret-minimization procedures for easy static graphs; on a complete graph, where movement constraints disappear, classical UCB-style methods and static graph-bandit algorithms attain sharper regret guarantees (\cite{auer2002finite,lattimore-szepesvari2020,zhang-johansson-li2023graphbandit}). The contribution here is a clean characterization of a dynamic regime in which the canonical walk supports sublinear regret via algorithms whose analysis separates the roles of learning and navigation.

\section{Exploration, Commitment, and Navigation}
\label{sec:algorithms_and_analysis}

% The central difficulty in this problem is the separation between identification and exploitation. In an ordinary multi-armed bandit, once the learner identifies the best arm it can play that arm immediately. Here, identification alone is not enough. The learner may know which arm is best and still be unable to reach it quickly because movement is local and the graph is changing. This creates a second task, navigation, that does not appear in classical bandits.

% Our algorithms all follow the same two-phase template. During exploration the learner moves locally according to an exploration kernel and collects samples. Once it stops exploring, it commits to the empirically best arm and begins a navigation phase. If the target arm is currently available, the learner moves to it immediately. Otherwise it uses the canonical walk as a fallback navigation rule until the target becomes available or is reached by the walk. Since staying put is always allowed, the learner can then remain at the target arm for the rest of the horizon.

The fundamental challenge in this setting is decoupling identification from exploitation. In classical bandits, identifying the optimal arm grants immediate and perpetual access to it. On a dynamic graph, however, local movement necessitates a distinct navigation phase: even with perfect knowledge of the best arm, the learner might be temporarily unable to reach it. This introduces a structural delay cost absent in unconstrained environments.

To address this, our algorithms adopt a phased explore-then-commit framework. During exploration, the learner moves according to a specific kernel to gather reward samples. Upon termination at a stopping time, the learner commits to the empirical leader and initiates navigation. If the target arm is not immediately reachable, the learner employs the canonical walk as a fallback navigation rule until the target is attained. Because self-loops are always available, the learner remains at the target arm for the remainder of the horizon.

% The present section develops the fixed-budget baseline LEX. The later confidence-based variant keeps the same exploration kernel and changes only the stopping rule. Reward-biased variants change the kernel itself and require additional arguments, so we defer them to later sections. 

We first develop LEX, a fixed-budget baseline that establishes the core analysis for this problem. While the later confidence-based variant retains the same exploration kernel and modifies only the stopping rule, reward-aware versions necessitate the modified kernels discussed in subsequent sections. The LEX framework comprises three elements: an exploration kernel $K_t$ for local movement, a stopping criterion, and a navigation rule toward the empirical leader. During navigation, the learner moves directly to the target arm whenever it is available, otherwise falling back to the canonical walk $U_t$. This direct-move rule improves upon pure random-walk navigation and ensures the policy simplifies correctly to the unconstrained case on a complete graph.

% \begin{algorithm}
% \caption{Exploration, commitment, and local navigation}
% \label{alg:template}
% \begin{algorithmic}[1]
% \Require horizon $T$, initial node $a_0$, exploration kernels $\{K_t\}_{t\ge 1}$, stopping rule $\mathrm{Stop}$
% \State For each $a\in A$, set $\varphi_0(a)\gets 0$ and $\hat\mu_0(a)\gets 0$
% \State $t\gets 1$
% \While{$t\le T$ and $\mathrm{Stop}(t-1,\{\varphi_{t-1}(a),\hat\mu_{t-1}(a)\}_{a\in A})=0$}
%     \State Observe the current local action set $L_t(a_{t-1})$
%     \State Sample $a_t\sim K_t(a_{t-1},\cdot)$ supported on $L_t(a_{t-1})$
%     \State Observe reward $r_t\sim \mathcal D(a_t)$
%     \State Update $\varphi_t(\cdot)$ and $\hat\mu_t(\cdot)$ using $(a_t,r_t)$
%     \State $t\gets t+1$
% \EndWhile
% \State Set $\hat a^\star \in \arg\max_{a\in A}\hat\mu_{t-1}(a)$
% \While{$t\le T$}
%     \State Observe the current local action set $L_t(a_{t-1})$
%     \If{$\hat a^\star\in L_t(a_{t-1})$}
%         \State $a_t \gets \hat a^\star$
%     \Else
%         \State Sample $a_t\sim U_t(a_{t-1},\cdot)$
%     \EndIf
%     \State Observe reward $r_t\sim \mathcal D(a_t)$
%     \State $t\gets t+1$
% \EndWhile
% \end{algorithmic}
% \end{algorithm}

\begin{algorithm}[t]
\caption{Exploration, commitment, and local navigation}
\label{alg:template}
\begin{algorithmic}[1]
\Require horizon $T$, initial node $a_0$, exploration kernels $\{K_t\}_{t\ge 1}$, stopping rule $\mathrm{Stop}$
\State $\forall a \in A: \varphi_0(a) \gets 0, \hat\mu_0(a) \gets 0$; $t \gets 1$
\While{$t \le T$ and $\mathrm{Stop}(t-1, \{\varphi_{t-1}, \hat\mu_{t-1}\}) = 0$}
    \State Sample $a_t \sim K_t(a_{t-1}, \cdot)$ on $L_t(a_{t-1})$;
    \State Observe $r_t \sim \mathcal{D}(a_t)$
    \State Update $(\varphi_t, \hat\mu_t)$ via $(a_t, r_t)$; $t \gets t+1$
\EndWhile
\State Commit to $\hat a^\star \in \arg\max_{a\in A} \hat\mu_{t-1}(a)$
\While{$t \le T$}
    \State $a_t \gets \hat a^\star$ if $\hat a^\star \in L_t(a_{t-1})$, else $a_t \sim U_t(a_{t-1}, \cdot)$
    \State Observe $r_t \sim \mathcal{D}(a_t)$; $t \gets t+1$
\EndWhile
\end{algorithmic}
\end{algorithm}

Our proofs follow a three-step structure: coverage, estimation, and navigation. First, under Definition~\ref{def:mixing-sequence}, the walk's distribution over arms approaches $\pi$ within $\tau_0$ steps from any starting node or time  
  --- no arm is systematically avoided regardless of where exploration begins. Second, rather than analyzing every step (which are time-correlated), we examine positions spaced $\tau_0$ steps apart; at each such checkpoint arm $a$ is     
  visited with probability at least $\pi(a)/2$, and a Chernoff bound gives a high-probability lower bound on visits to every arm. Third, once each arm has enough visits, Hoeffding's inequality identifies $a^\star$. The navigation analysis
   re-applies the same argument to the random graph suffix after exploration ends. Letting $\pi_* \triangleq \min_{a \in A}\pi(a)$, Appendix~\ref{app:technical-lemmas} shows that \begin{equation}\label{eq:tau-mix-main}\tau_{\mathrm{mix}}(\varepsilon)\triangleq W+\left\lceil \frac{1}{\rho\gamma}\Bigl[\log\frac{1}{2\varepsilon\sqrt{\pi_*}}\Bigr]_+\right\rceil\end{equation} is a valid total-variation mixing time, uniform over all initial states and times.

\subsection{LEX: fixed-budget lazy exploration}
\label{subsec:lex_cblex}

LEX employs the canonical walk for exploration, setting $K_t = U_t$ for all $t$. Following a fixed budget $T_{\exp}$, the learner commits to the empirical leader. Because $T_{\exp}$ depends on both reward gaps and the structural parameters of Definition \ref{def:mixing-sequence}, LEX serves as a tuned baseline rather than an adaptive policy. Its primary utility is conceptual, as it isolates the statistical cost of learning from the structural cost of reaching the target arm. A key property of LEX is that its trajectory remains independent of observed rewards. The sequence of visited arms, and consequently the sample counts, is governed solely by the graph sequence and walk randomness. This independence ensures that the subsequent estimation step is mathematically transparent. The fixed-budget nature of LEX implies that identification guarantees are informative only when the horizon $T$ exceeds $T_{\exp}$. If $T < T_{\exp}$, the algorithm explores until the horizon concludes, where the trivial bound $R(T) \le T$ applies. This short-horizon case is incorporated into the expected-regret corollary, while the main theorem is stated for the regime $T \ge T_{\exp}$.

Theorem~\ref{thm:lex_high_prob} shows that with a budget of $\tau_0$ mixing-time blocks, each arm accumulates enough visits to concentrate its empirical mean and the optimal arm is identified correctly with high probability. The budget is governed by two quantities: the mixing overhead $\tau_0$ and the per-arm statistical difficulty $\Psi_{\mathrm{LEX}}$, both defined in the theorem.

\begin{theorem}[\textsc{Lex} under common-stationary mixing]
\label{thm:lex_high_prob}
Assume the graph sequence is $(W,\rho,\gamma,\pi)$-mixing for the canonical walk. Let $\tau_0 \triangleq W+\lceil \frac{2}{\rho\gamma}\log\frac{1}{\pi_*}\rceil$ and define the complexity:
\[
\Psi_{\mathrm{LEX}}(\delta) \triangleq \max\left\{ \frac{\log(2n/\delta)}{\pi(a^\star)\Delta_{\min}^2}, \max_{a\neq a^\star}\frac{\log(2n/\delta)}{\pi(a)\Delta(a)^2} \right\}.
\]
For a universal constant $C>0$, if \textsc{Lex} uses an exploration budget $T_{\exp} \triangleq \tau_0 \lceil C\Psi_{\mathrm{LEX}}(\delta) \rceil \le T$, then with probability at least $1-\delta$, the algorithm correctly identifies the optimal arm ($\hat a^\star=a^\star$) and the learning regret satisfies $R_{\mathrm{learn}}(T) \le T_{\exp}$.
\end{theorem}

% The proof is short once the mixing lemma is available. We look only at the positions of the walk every $\tau_0$ steps. By construction, each thinned position is within total variation distance $\pi_*/2$ of $\pi$, no matter where the previous block started. Therefore every arm $a$ is visited at each thinned time with conditional probability at least $\pi(a)/2$. A one-sided Chernoff bound for adapted Bernoulli variables then shows that the number of such visits is linear in $m\pi(a)$ with high probability. Since the actual number of samples of arm $a$ is at least as large as the thinned \sourav{need to explain this term briefly, because it seems abrupt.} count, Hoeffding's inequality then concentrates the empirical mean of every arm, including the optimal arm. The resulting sample requirement is captured exactly by $\Psi_{\mathrm{LEX}}(\delta)$. The full proof is given in Appendix~\ref{app:lex-proof}.
The proof follows directly from the mixing lemma by analyzing the walk at fixed intervals of length $\tau_0$. This thinning procedure, which considers only every $\tau_0$-th step to decouple temporal dependencies, ensures that each sampled position is within total variation distance $\pi_*/2$ of $\pi$, regardless of the preceding block's state. Consequently, every arm $a$ is visited at these thinned steps with conditional probability at least $\pi(a)/2$. A one-sided Chernoff bound for adapted Bernoulli variables confirms that the number of such visits is proportional to $m\pi(a)$ with high probability. Because the total samples for arm $a$ necessarily exceed this thinned count, Hoeffding's inequality ensures the concentration of the empirical means for all arms. This sample requirement is captured exactly by $\Psi_{\mathrm{LEX}}(\delta)$, with the full proof provided in Appendix \ref{app:lex-proof}.

% The next lemma controls the cost of moving to the chosen arm after exploration stops.

% \begin{lemma}[Navigation under the canonical walk]
% \label{lem:navigation-regret}
% Assume the same conditions as in Theorem~\ref{thm:lex_high_prob}. If $\hat a^\star=a^\star$, then the navigation phase of Algorithm~\ref{alg:template} satisfies
% \[
% \mathbb E[R_{\mathrm{nav}}]
% \le
% \Delta_{\max}\,\mathbb E[\tau_{\mathrm{hit}}(a^\star)]
% \le
% \frac{2\Delta_{\max}\tau_0}{\pi(a^\star)},
% \]
% where $\Delta_{\max}\triangleq\max_{a\in A}\Delta(a)$ and $\tau_{\mathrm{hit}}(a^\star)$ is the first time the navigation phase reaches $a^\star$.
% \end{lemma}

% The navigation bound is independent of the horizon. Its proof compares the actual navigation phase, which freezes after reaching the target, to a pure canonical walk run over the same time block. This is the structural part of the regret decomposition: once the best arm has been identified correctly, the remaining cost is only the time needed to reach it under local movement. Combining Theorem~\ref{thm:lex_high_prob} with Lemma~\ref{lem:navigation-regret} gives the expected regret bound.
The following lemma bounds the structural cost of navigating to the identified arm once exploration concludes.\begin{lemma}[Navigation under the canonical walk]\label{lem:navigation-regret}Assume the conditions of Theorem \ref{thm:lex_high_prob}. If $\hat a^\star=a^\star$, the navigation phase of Algorithm \ref{alg:template} satisfies$$\mathbb E[R_{\mathrm{nav}}]
\le
\Delta_{\max}\,\mathbb E[\tau_{\mathrm{hit}}(a^\star)]
\le
\frac{2\Delta_{\max}\tau_0}{\pi(a^\star)},$$where $\Delta_{\max}\triangleq\max_{a\in A}\Delta(a)$ and $\tau_{\mathrm{hit}}(a^\star)$ is the first time the navigation phase reaches $a^\star$.\end{lemma}The navigation bound is independent of the horizon $T$, illustrating the core regret decomposition. The proof works by coupling the navigation phase to a pure canonical walk started from the same state and time. This isolates the structural cost as the expected time required to reach $a^\star$ under local movement constraints. The structural part of the argument is then a block-wise endpoint bound for the canonical walk obtained from Lemma~\ref{lem:uniform-mixing}. Once the best arm has been identified correctly, the remaining cost is therefore just the time needed for a canonical walk, up to this domination, to reach the target under local movement.

Combining Theorem \ref{thm:lex_high_prob} with Lemma \ref{lem:navigation-regret} yields the total expected regret guarantee.

\begin{corollary}[\textsc{Lex} Expected Regret]
\label{cor:lex_expected}
Under Theorem \ref{thm:lex_high_prob} with $\delta = T^{-2}$, \textsc{Lex} achieves $\mathbb{E}[R(T)] = O(\tau_0 \Psi_{\mathrm{LEX}}(T^{-2}) + \tau_0/\pi(a^\star))$. For the $d$-regular case, let $\Gamma = W + \frac{\log n}{\rho\gamma}$; the regret simplifies to:
\begin{equation}
\mathbb{E}[R(T)] = O\left( \Gamma n \left[ \Delta_{\min}^{-2} \log(nT) + 1 \right] \right).
\end{equation}
\end{corollary}

In the regular case where $\pi$ is uniform, let $\Gamma = W+\frac{\log n}{\rho\gamma}$. The bound simplifies to:
\begin{equation}
\mathbb{E}[R(T)] = O\left( \Gamma \frac{n\log(nT)}{\Delta_{\min}^2} \right) + O\left( \Gamma n \right).
\end{equation}
% This bound decouples the two primary sources of regret: the statistical complexity of identification and the structural overhead of navigation. The learning term scales with the inverse suboptimality gaps and the inverse stationary masses of all arms, reflecting that even $a^\star$ requires sufficient visitation for empirical mean concentration. The navigation term is purely structural and independent of the horizon $T$. In the regular case, where $\pi$ is uniform, the bound simplifies to:$$\mathbb E[R(T)]
% =
% O\!\left(
% \left[
% W+\frac{1}{\rho\gamma}\log n
% \right]
% \frac{n\log(nT)}{\Delta_{\min}^2}
% \right)
% +
% O\!\left(
% \left[
% W+\frac{1}{\rho\gamma}\log n
% \right]n
% \right).$$We next replace the fixed budget of \textsc{Lex} with a data-dependent stopping rule while maintaining the same exploration kernel to develop a confidence-based variant.
The bound decouples the statistical complexity of identification from the structural overhead of navigation. The former scales with inverse gaps and stationary masses (ensuring $a^\star$ is sufficiently sampled), while the latter remains a $T$-independent constant. In the $d$-regular case, letting $\Gamma \triangleq W + \frac{\log n}{\rho\gamma}$ represent the mixing overhead, the regret simplifies to:$$\mathbb{E}[R(T)] = O\left( \Gamma n \left[ \Delta_{\min}^{-2} \log(nT) + 1 \right] \right).$$We now replace the fixed budget with a data-dependent stopping rule to develop a confidence-based variant.

\subsection{CB-LEX: confidence-based stopping}
\label{subsec:cblex}

% LEX isolates the statistical content of the problem, but it still requires the exploration budget to be chosen from the unknown reward gaps. CB-LEX removes that oracle dependence. The walk itself does not change: during exploration the learner continues to follow the canonical kernel $U_t$. The only change is the rule that decides when enough evidence has been collected. Relative to prior work, the stopping certificate is also corrected \sourav{right word?} here: the leader must dominate all competing upper confidence bounds, not only the empirical runner-up.

% At a high level, the algorithm keeps a confidence interval for every arm and stops as soon as one arm is separated from all competitors by confidence bounds. In the present setting, this comparison has to be made against every other arm, not only against the empirical runner-up. The reason is that the sample counts of the arms can be highly uneven. An arm that has been visited only a few times can have a low empirical mean and still carry a very wide confidence interval. A stopping rule that compares the leader only with the empirical second-best arm can therefore stop too early. The correct certificate is that the lower confidence bound of the current leader dominates the largest upper confidence bound among all other arms. 
While LEX requires an oracle exploration budget, CB-LEX introduces a data-dependent stopping rule while retaining the canonical kernel $U_t$. A key distinction is the stopping certificate: the empirical leader's lower confidence bound must dominate the upper confidence bounds (UCBs) of all competitors, rather than just the empirical runner-up. 
%\sourav{this gives us the same condition as the one in the prev work.}
This refinement is essential because sample counts can be highly uneven; a poorly sampled arm might have a low empirical mean but a wide confidence interval that still overlaps with the leader. Comparing only against the runner-up would risk premature commitment before these wide intervals are sufficiently narrowed.

For each arm $a \in A$ at time $t$, let $\hat\mu_t(a)$ be the empirical mean and $\varphi_t(a)$ the visitation count. We define the confidence width $w_t(a) \triangleq \sqrt{\ln(4nT/\delta) / (2\varphi_t(a))}$ if $\varphi_t(a) \ge 1$, and $w_t(a) \triangleq 1$ otherwise. The confidence bounds are $\operatorname{LCB}_t(a) \triangleq \hat\mu_t(a) - w_t(a)$ and $\operatorname{UCB}_t(a) \triangleq \hat\mu_t(a) + w_t(a)$. Let $b_t \in \arg\max_{a \in A} \hat\mu_t(a)$ be the current empirical leader. CB-LEX terminates exploration at the stopping time:
\begin{equation}
\label{eq:cblex-stop}
\sigma \triangleq \min \{ t \le T \mid \operatorname{LCB}_t(b_t) \ge \max_{a \neq b_t} \operatorname{UCB}_t(a) \}
\end{equation}
with $\sigma \triangleq T$ if the condition is never met. If $\sigma < T$, the algorithm commits to $\hat a^\star \triangleq b\sigma$ and initiates the same navigation phase used in LEX. This stopping rule ensures the leader is separated from all competitors, which is essential because sample counts can be highly uneven; a poorly sampled arm might have a low empirical mean but a wide confidence interval that still overlaps with the leader.

% The exploration phase of CB-LEX is summarized in Algorithm~\ref{alg:cblex}. The navigation phase is exactly the same as in Algorithm~\ref{alg:template}, so we do not repeat it here. The widths in \eqref{eq:cblex-width} depend on the horizon only through the logarithm. A fully horizon-free variant can be obtained by the standard replacement $T\mapsto (t+1)^2$, but the horizon-dependent form keeps the notation lighter and is enough for the present paper.

The exploration phase of CB-LEX is summarized in Algorithm \ref{alg:cblex}. Since the navigation logic remains identical to Algorithm \ref{alg:template}, it is omitted for brevity. The confidence widths $w_t(a)$ depend on the horizon $T$ only through the logarithm. While a fully horizon-free variant can be obtained by the standard replacement $T \mapsto (t+1)^2$, we retain the horizon-dependent form to keep the notation streamlined for the current analysis. As established in Theorem \ref{thm:cblex}, the adaptive stopping rule preserves the regret order of the LEX baseline. Although the logarithmic terms differ slightly to ensure a uniform confidence event over the horizon, both algorithms achieve the same asymptotic rate under the standard choice $\delta=T^{-2}$.

% \begin{algorithm}
% \caption{CB-LEX exploration phase}
% \label{alg:cblex}
% \begin{algorithmic}[1]
% \Require horizon $T$, confidence parameter $\delta$, initial node $a_0$
% \State For each $a\in A$, set $\varphi_0(a)\gets 0$ and $\hat\mu_0(a)\gets 0$
% \For{$t=1$ {\bf to} $T$}
%     \State Observe the current local action set $L_t(a_{t-1})$
%     \State Sample $a_t\sim U_t(a_{t-1},\cdot)$ supported on $L_t(a_{t-1})$
%     \State Observe reward $r_t\sim \mathcal D(a_t)$
%     \State Update $\varphi_t(\cdot)$ and $\hat\mu_t(\cdot)$ using $(a_t,r_t)$
%     \State For each $a\in A$, compute $w_t(a)$ from \eqref{eq:cblex-width}
%     \State Let $b_t\in \arg\max_{a\in A}\hat\mu_t(a)$
%     \If{$\hat\mu_t(b_t)-w_t(b_t)\ge \max_{a\neq b_t}\bigl(\hat\mu_t(a)+w_t(a)\bigr)$}
%         \State Set $\sigma\gets t$ and $\hat a^\star\gets b_t$
%         \State Switch to the navigation phase of Algorithm~\ref{alg:template}
%     \EndIf
% \EndFor
% \State If the stopping condition never triggered, set $\sigma\gets T$ and terminate
% \end{algorithmic}
% \end{algorithm}

\begin{algorithm}[t]
\caption{CB-LEX exploration phase}
\label{alg:cblex}
\begin{algorithmic}[1]
\Require horizon $T$, confidence parameter $\delta$, initial node $a_0$.
\State $\forall a \in A: \varphi_0(a) \gets 0, \hat\mu_0(a) \gets 0$
\For{$t = 1$ \textbf{to} $T$}
    \State Sample $a_t \sim U_t(a_{t-1}, \cdot)$ on $L_t(a_{t-1})$
    \State Observe $r_t \sim \mathcal{D}(a_t)$
    \State Update $(\varphi_t, \hat\mu_t)$ via $(a_t, r_t)$; $b_t \in \arg\max_{a \in A} \hat\mu_t(a)$
    \If{$\operatorname{LCB}_t(b_t) \ge \max_{a \neq b_t} \operatorname{UCB}_t(a)$}
        \State $\sigma \gets t, \hat{a}^\star \gets b_t$; \textbf{Break} to navigation phase~(\ref{alg:template})
    \EndIf
\EndFor
\State \textbf{If} $t > T$: $\sigma \gets T$ and \textbf{Terminate}
\end{algorithmic}
\end{algorithm}

\begin{theorem}[CB-LEX under common-stationary mixing]
\label{thm:cblex}
Assume the graph sequence is $(W,\rho,\gamma,\pi)$-mixing and let $\tau_0$ be as in Theorem 1. Define the complexity:
\begin{equation}
\label{eq:psi_cb}
\Psi_{\mathrm{CB}}(\delta) \triangleq \max \left\{ \frac{\log(4nT/\delta)}{\pi(a^\star)\Delta_{\min}^2}, \max_{a \neq a^\star} \frac{\log(4nT/\delta)}{\pi(a)\Delta(a)^2} \right\}.
\end{equation}
For a constant $C>0$, if $T \ge \tau_0 \lceil C\Psi_{\mathrm{CB}}(\delta) \rceil \triangleq T_{\mathrm{nom}}$, then with probability at least $1-\delta$: $\hat a^\star = a^\star$, the algorithm stops at $\sigma \le T_{\mathrm{nom}}$, and $R_{\mathrm{learn}}(T) \le T_{\mathrm{nom}}$.
\end{theorem}

% The proof uses exactly the same structural ingredients as the proof of LEX. The new step is to couple CB-LEX to a perpetual exploration walk that continues to follow the canonical kernel even after the stopping condition has become true. Up to the stopping time, the real algorithm and the perpetual walk are identical. It is therefore enough to show that the stopping condition must hold by a deterministic time on the perpetual walk. The mixing lemma and the visitation lemma from Appendix~\ref{app:technical-lemmas} provide the necessary sample counts at that deterministic time, and a uniform confidence event converts those counts into separation of the confidence intervals. The full argument is given in Appendix~\ref{app:cblex-proof}.

The proof shares the structural ingredients of the LEX analysis, with the addition of a coupling argument. We couple CB-LEX to a perpetual exploration walk that continues to follow the canonical kernel regardless of the stopping condition. Up to the time $\sigma$, the actual trajectory and the perpetual walk are identical. Consequently, it suffices to show that the stopping condition is satisfied by a deterministic time on the perpetual walk. The mixing and visitation lemmas (Appendix \ref{app:technical-lemmas}) provide the necessary sample counts at that time, and a uniform confidence event ensures the separation of intervals. The full argument is provided in Appendix \ref{app:cblex-proof}.

% \begin{corollary}[Expected regret of CB-LEX]
% \label{cor:cblex_expected}
% Under the conditions of Theorem~\ref{thm:cblex}, choose $\delta=T^{-2}$ and define
% \[
% T_{\mathrm{CB}}^{\mathrm{nom}}
% \triangleq
% \tau_0\left\lceil C\,\Psi_{\mathrm{CB}}(T^{-2})\right\rceil.
% \]
% Then running CB-LEX on horizon $T$ gives
% \[
% \mathbb E[R(T)]
% \le
% T_{\mathrm{CB}}^{\mathrm{nom}}
% +
% \frac{2\Delta_{\max}\tau_0}{\pi(a^\star)}
% +
% 1.
% \]
% Equivalently,
% \[
% \mathbb E[R(T)]
% =
% O\!\left(
% \tau_0
% \max\left\{
% \frac{\log(4nT^3)}{\pi(a^\star)\Delta_{\min}^2},
% \max_{a\neq a^\star}\frac{\log(4nT^3)}{\pi(a)\Delta(a)^2}
% \right\}
% \right)
% +
% O\!\left(\frac{\tau_0}{\pi(a^\star)}\right),
% \]
% and in particular $\mathbb E[R(T)]=o(T)$.
% \end{corollary}
% \begin{corollary}[Expected regret of \textsc{CB-LEX}]
% \label{cor:cblex_expected}
% Under the conditions of Theorem~\ref{thm:cblex}, setting $\delta=T^{-2}$ and defining $T_{\mathrm{CB}}^{\mathrm{nom}} \triangleq \tau_0\lceil C\Psi_{\mathrm{CB}}(T^{-2})\rceil$ yields $\mathbb E[R(T)] \le T_{\mathrm{CB}}^{\mathrm{nom}} + \frac{2\Delta_{\max}\tau_0}{\pi(a^\star)} + 1$. Equivalently, we achieve sublinear regret $\mathbb E[R(T)]=o(T)$ bounded asymptotically by:
% \[
% \mathbb E[R(T)] = O\!\left( \tau_0 \max\left\{ \frac{\log(nT)}{\pi(a^\star)\Delta_{\min}^2}, \max_{a\neq a^\star}\frac{\log(nT)}{\pi(a)\Delta(a)^2} \right\} + \frac{\tau_0}{\pi(a^\star)} \right).
% \]
% \end{corollary}

\begin{corollary}[Expected regret of CB-LEX]
\label{cor:cblex_expected}
Under the conditions of Theorem \ref{thm:cblex}, setting $\delta=T^{-2}$ yields sublinear regret $\mathbb E[R(T)]=o(T)$. The expected regret is bounded by:
\begin{equation}
\mathbb E[R(T)] = O \left( \tau_0 \Psi_{\mathrm{CB}}(T^{-2}) + \frac{\tau_0}{\pi(a^\star)} \right).
\end{equation}
In the $d$-regular case, letting $\Gamma \triangleq W + \frac{\log n}{\rho\gamma}$, this simplifies to:
\begin{equation}
\mathbb{E}[R(T)] = O \left( \Gamma n \left[ \Delta_{\min}^{-2} \log(nT) + 1 \right] \right).
\end{equation}
\end{corollary}

% The main point of the corollary is that the confidence-based rule removes the need to specify the exploration length in advance without creating a new asymptotic penalty. Under the walk-centric mixing condition of Definition~\ref{def:mixing-sequence}, the same visitation theorem that justifies the fixed budget of LEX also implies that the confidence intervals become disjoint after the same order of time. What changes is the way that fact is detected. In LEX the budget is supplied externally; in CB-LEX it is certified from the data. The only genuinely new ingredient in the proof is the coupling to the perpetual exploration walk, which lets us analyze a random stopping rule using the deterministic-time bounds already established for the canonical walk.
Crucially, the confidence-based rule eliminates the requirement for a pre-specified exploration budget without incurring an asymptotic penalty. Under the mixing condition of Definition~\ref{def:mixing-sequence}, the same visitation guarantees that justify the LEX budget ensure that confidence intervals become disjoint within the same order of time. The fundamental distinction lies in the termination mechanism: LEX utilizes an external budget, whereas CB-LEX certifies stopping through observed data. The primary technical novelty in the analysis is the coupling to a perpetual exploration walk, which enables the evaluation of a random stopping rule using the deterministic-time bounds previously established for the canonical walk.

\subsection{RALEX: reward-aware exploration with a canonical floor}
\label{subsec:ralex}

While CB-LEX uses data only for its stopping rule, its movement remains reward-indifferent. RALEX introduces reward-awareness by biasing transitions toward promising arms. The difficulty is that a purely reward-driven kernel often fails to preserve the stationary distribution $\pi$ required by Definition 1. RALEX instead employs a mixture that retains a fixed floor of canonical movement.At round $t$, the learner uses the empirical means $\hat\mu_{t-1}$ and widths $w_{t-1}$ to form the local optimistic score $\xi_{t-1}(a) \triangleq \hat\mu_{t-1}(a) + w_{t-1}(a)$. We define a softmax proposal on the neighborhood $L_t(a_{t-1})$:\begin{equation}\label{eq:ralex-softmax}S_t(a_{t-1}, v) \triangleq \frac{e^{\lambda \xi_{t-1}(v)}}{\sum_{z \in L_t(a_{t-1})} e^{\lambda \xi_{t-1}(z)}}, \quad v \in L_t(a_{t-1}),\end{equation}where $\lambda \ge 0$ controls the bias. To ensure mixing, RALEX defines the weight $\epsilon_t \triangleq \max\{\epsilon_0, 2\bar w_{t-1}/(1 + 2\bar w_{t-1})\}$, where $\bar w_{t-1} \triangleq \max_{a \in A} w_{t-1}(a)$ and $\epsilon_0 \in (0, 1]$ is a fixed floor. The exploration kernel is:\begin{equation}\label{eq:ralex-kernel}K_t(a_{t-1}, v) \triangleq \epsilon_t U_t(a_{t-1}, v) + (1 - \epsilon_t) S_t(a_{t-1}, v).\end{equation}

RALEX adopts the same stopping rule as CB-LEX, terminating exploration when the leader’s lower bound separates from all competing upper bounds.

% In Algorithm~\ref{alg:ralex}, the first part of each round computes the optimistic proposal from the previous round's estimates and widths, and then chooses how much of the step must remain on the canonical walk. After the move and the reward update, the algorithm recomputes the widths using the new sample counts and checks the same confidence certificate used by CB-LEX. The exploration kernel is therefore the only algorithmic change, but the proof has to account for the fact that the visit counts are now reward-dependent.
Algorithm \ref{alg:ralex} details the RALEX exploration phase. Each round begins by calculating the optimistic proposal $\xi_{t-1}$ from the preceding round's estimates and determining the canonical mixing weight $\epsilon_t$. Following the state transition and subsequent reward update, the algorithm refreshes the confidence widths and evaluates the stopping certificate established for CB-LEX. Although the exploration kernel is the primary algorithmic departure, the theoretical analysis must rigorously address the reward-dependence of the resulting visitation counts.

% \begin{algorithm}
% \caption{RALEX exploration phase}
% \label{alg:ralex}
% \begin{algorithmic}[1]
% \Require horizon $T$, confidence parameter $\delta$, bias parameter $\lambda\ge 0$, floor $\varepsilon_0\in(0,1]$, initial node $a_0$
% \State For each $a\in A$, set $\varphi_0(a)\gets 0$ and $\hat\mu_0(a)\gets 0$
% \For{$t=1$ {\bf to} $T$}
%     \State Observe the current local action set $L_t(a_{t-1})$
%     \State For each $a\in A$, compute $w_{t-1}(a)$ from \eqref{eq:cblex-width} and set $\xi_{t-1}(a)=\hat\mu_{t-1}(a)+w_{t-1}(a)$
%     \State Form the local softmax proposal $S_t(a_{t-1},\cdot)$ on $L_t(a_{t-1})$ according to \eqref{eq:ralex-softmax}
%     \State Compute $\bar w_{t-1}=\max_{a\in A}w_{t-1}(a)$ and set $\varepsilon_t$ according to \eqref{eq:ralex-eps-data}--\eqref{eq:ralex-eps}
%     \State Define $K_t(a_{t-1},\cdot)=\varepsilon_t U_t(a_{t-1},\cdot)+(1-\varepsilon_t)S_t(a_{t-1},\cdot)$
%     \State Sample $a_t\sim K_t(a_{t-1},\cdot)$
%     \State Observe reward $r_t\sim \mathcal D(a_t)$
%     \State Update $\varphi_t(\cdot)$ and $\hat\mu_t(\cdot)$ using $(a_t,r_t)$
%     \State For each $a\in A$, recompute $w_t(a)$ from \eqref{eq:cblex-width}
%     \State Let $b_t\in \arg\max_{a\in A}\hat\mu_t(a)$
%     \If{$\hat\mu_t(b_t)-w_t(b_t)\ge \max_{a\neq b_t}\bigl(\hat\mu_t(a)+w_t(a)\bigr)$}
%         \State Set $\sigma\gets t$ and $\hat a^\star\gets b_t$
%         \State Switch to the navigation phase of Algorithm~\ref{alg:template}
%     \EndIf
% \EndFor
% \State If the stopping condition never triggered, set $\sigma\gets T$ and terminate
% \end{algorithmic}
% \end{algorithm}

\begin{algorithm}[t]
\caption{RALEX exploration phase}
\label{alg:ralex}
\begin{algorithmic}[1]
\Require $T, \delta, \lambda, \varepsilon_0, a_0$
\State $\forall a \in A: \varphi_0(a) \gets 0, \hat\mu_0(a) \gets 0$
\For{$t = 1$ \textbf{to} $T$}
    \State Compute widths $w_{t-1}$ and scores $\xi_{t-1}$
    \State Form $S_t$ via \eqref{eq:ralex-softmax}
    \State Set weight $\epsilon_t$ and define $K_t$ according to \eqref{eq:ralex-kernel}
    \State Sample $a_t \sim K_t(a_{t-1}, \cdot)$
    \State Observe reward $r_t \sim \mathcal{D}(a_t)$.
    \State Update $(\varphi_t, \hat\mu_t)$ via $(a_t, r_t)$; $b_t \in \arg\max_{a \in A} \hat\mu_t(a)$
    \If{$\operatorname{LCB}_t(b_t) \ge \max_{a \neq b_t} \operatorname{UCB}_t(a)$}
        \State $\sigma \gets t, \hat{a}^\star \gets b_t$; \textbf{Break} to navigation phase (\ref{alg:template})
    \EndIf
\EndFor
\State \textbf{If} $t > T$: $\sigma \gets T$ and \textbf{Terminate}
\end{algorithmic}
\end{algorithm}

% The fixed floor in \eqref{eq:ralex-eps} is the key structural ingredient. It implies the pointwise lower bound
% \begin{equation}
% \label{eq:ralex-floor}
% K_t(u,v)\ge \varepsilon_0 U_t(u,v)
% \qquad\text{for every }t\ge 1\text{ and every }u,v\in A.
% \end{equation}
% The worst-case analysis of RALEX uses only this inequality. In particular, the theorem below does not need any cap on $\lambda$. The price of this generality is that the resulting bound is conservative: the proof treats the biased part of the kernel as completely arbitrary except for the guaranteed canonical floor.
The deterministic floor $\epsilon_0$ in \eqref{eq:ralex-kernel} provides the key structural guarantee for RALEX. It implies the pointwise lower bound:
\begin{equation}
\label{eq:ralex-floor}
K_t(u,v) \ge \epsilon_0 U_t(u,v), \quad \forall t \ge 1, \ u,v \in A.
\end{equation}
The worst-case analysis relies exclusively on this inequality, ensuring the subsequent theorem holds for any bias parameter $\lambda \ge 0$. While this generality leads to a conservative bound; treating the biased component as an arbitrary perturbation; it preserves the rigorous structural properties established for the canonical walk.

% \begin{theorem}[RALEX under common-stationary sliding-window mixing]
% \label{thm:ralex}
% Assume that the realized graph sequence is $(W,\rho,\gamma,\pi)$-mixing for the canonical walk. Fix $\lambda\ge 0$ and $\varepsilon_0\in(0,1]$, let $\tau_0$ be the quantity defined in Theorem~\ref{thm:lex_high_prob}, and set $\vartheta_0\triangleq\frac{1}{2}\,\varepsilon_0^{\tau_0}.$
% Define
% \[
% \Psi_{\mathrm{RA}}(\delta)
% \triangleq
% \max\left\{
% \frac{\log(4nT/\delta)}{\vartheta_0\pi(a^\star)\Delta_{\min}^2},
% \max_{a\neq a^\star}\frac{\log(4nT/\delta)}{\vartheta_0\pi(a)\Delta(a)^2}
% \right\}.
% \]
% There exists a universal constant $C>0$ such that the following holds. If
% \[
% T_{\mathrm{RA}}^{\mathrm{nom}}
% \triangleq
% \tau_0\left\lceil C\,\Psi_{\mathrm{RA}}(\delta)\right\rceil
% \le T,
% \]
% then $\Pr\!\left(
% \sigma\le T_{\mathrm{RA}}^{\mathrm{nom}}
% \text{ and }
% \hat a^\star=a^\star
% \right)
% \ge 1-\delta.$
% Consequently, $ R_{\mathrm{learn}}(T)\le T_{\mathrm{RA}}^{\mathrm{nom}}$ with probability at least $1-\delta$.
% \end{theorem}

\begin{theorem}[RALEX under common-stationary mixing]
\label{thm:ralex}
Assume the graph sequence is $(W,\rho,\gamma,\pi)$-mixing. For bias $\lambda \ge 0$ and floor $\varepsilon_0 \in (0,1]$, let $\tau_0$ be the mixing time from Theorem 1 and $\vartheta_0 \triangleq \frac{1}{2} \varepsilon_0^{\tau_0}$. Define the complexity:
\begin{equation}
\label{eq:psi_ra}
\Psi_{\mathrm{RA}}(\delta) \triangleq \max \left\{ \frac{\log(4nT/\delta)}{\vartheta_0 \pi(a^\star) \Delta_{\min}^2}, \max_{a \neq a^\star} \frac{\log(4nT/\delta)}{\vartheta_0 \pi(a) \Delta(a)^2} \right\}.
\end{equation}
For a constant $C>0$, if $T \ge T_{\mathrm{nom}} \triangleq \tau_0 \lceil C \Psi_{\mathrm{RA}}(\delta) \rceil$, then with probability at least $1-\delta$: $\hat a^\star = a^\star$, the algorithm stops at $\sigma \le T_{\mathrm{nom}}$, and $R_{\mathrm{learn}}(T) \le T_{\mathrm{nom}}$.
\end{theorem}

The analysis of RALEX introduces two primary technical refinements beyond the CB-LEX framework. The first is statistical: because the exploration kernel is reward-dependent, sample counts are no longer independent of the reward history. We address this in Appendix \ref{app:ralex-proof} using a uniform confidence lemma for adaptive sampling derived from a reward-table construction. The second is structural: since RALEX deviates from the canonical walk, the standard visitation lemma for $U_t$ does not apply directly. Instead, we show that each block of $\tau_0$ rounds behaves as a canonical walk with probability at least $\varepsilon_0^{\tau_0}$, yielding a thinned visitation guarantee with success rate $\vartheta_0\pi(a)$. With these ingredients, the stopping-time analysis follows the CB-LEX pattern.The resulting bound is conservative compared to CB-LEX, incurring a factor $\vartheta_0^{-1} = 2\varepsilon_0^{-\tau_0}$ by relying strictly on the canonical component from \eqref{eq:ralex-floor}. This constitutes the theoretical price for permitting the biased kernel component to depend arbitrarily on observations. At $\varepsilon_0=1$, RALEX simplifies to CB-LEX and the theorem recovers the CB-LEX rate. Notably, this result generalizes to any local proposal kernel mixed with the same canonical floor.\begin{corollary}[Expected regret of RALEX]\label{cor:ralex_expected}Under the conditions of Theorem \ref{thm:ralex}, let $\delta=T^{-2}$ and $T_{\mathrm{nom}} \triangleq \tau_0\lceil C\Psi_{\mathrm{RA}}(T^{-2})\rceil$. The expected regret of RALEX satisfies:\begin{equation}\mathbb E[R(T)] \le T_{\mathrm{nom}} + \frac{2\Delta_{\max}\tau_0}{\pi(a^\star)} + 1.\end{equation}Asymptotically, RALEX achieves sublinear regret $\mathbb E[R(T)] = o(T)$ bounded by:\begin{equation}\mathbb E[R(T)] = O \left( \tau_0 \Psi_{\mathrm{RA}}(T^{-2}) + \frac{\tau_0}{\pi(a^\star)} \right).\end{equation}\end{corollary}

% RALEX therefore occupies a different point in the design space than LEX and CB-LEX. The earlier algorithms keep the exploration walk exactly equal to the canonical walk and obtain the sharpest worst-case guarantee available under Definition~\ref{def:mixing-sequence}. RALEX trades some of that worst-case sharpness for a movement rule that actively favors locally promising arms. Theorem~\ref{thm:ralex} is a worst-case statement: it shows that the reward-aware bias can be added without losing sublinear regret, but it does not by itself certify that the bias helps. To obtain a genuine improvement one needs an additional condition saying that the optimistic proposal moves extra probability mass toward the arms that actually determine the stopping time. The next result isolates exactly that mechanism.
RALEX represents a distinct trade-off within the algorithmic design space. While LEX and CB-LEX achieve the sharpest worst-case guarantees by adhering strictly to the canonical walk, RALEX sacrifices a portion of this theoretical sharpness to prioritize transitions toward locally promising arms. Theorem \ref{thm:ralex} serves as a worst-case baseline; it establishes that reward-aware bias is compatible with sublinear regret, though it does not formally certify a performance gain over the baseline. A genuine improvement requires an additional condition: the optimistic proposal must successfully allocate probability mass toward the specific arms that determine the stopping time. The following result isolates this mechanism.

\subsection{When reward awareness yields a provable gain}
\label{subsec:ralex-gain}

% Theorem~\ref{thm:ralex} shows that the optimistic bias can be added without destroying the canonical exploration guarantee, but it does not yet identify a regime in which the bias provably helps.

% The relevant object is the perpetual exploration process. Let $\{\widetilde X_t\}_{t\ge 0}$ be the process that keeps running the RALEX exploration kernel forever, without ever switching to navigation. Fix a deterministic burn-in time $t_{\mathrm{gain}}\in\{0,\dots,T\}$ and define
% \[
% \tau_1
% \triangleq
% 1+\tau_{\mathrm{mix}}\!\left(\frac{\pi_*}{4}\right)
% =
% 1+W+\left\lceil \frac{1}{\rho\gamma}\log\frac{2}{\pi_*^{3/2}}\right\rceil,
% \qquad
% \vartheta_1\triangleq\frac{1}{2}\,\varepsilon_0^{\tau_1}.
% \]
% For every block index $j\ge 1$, write $s_j\triangleqt_{\mathrm{gain}}+(j-1)\tau_1,
% \qquad
% \widetilde{\mathcal G}_{j-1}\triangleq\sigma\!\left(\widetilde{\mathcal F}_{s_j}\right).$ The quantity $\Pr\!\left(\widetilde X_{s_j+\tau_1}=a\mid \widetilde{\mathcal G}_{j-1}\right)$
% is the actual probability that the $j$-th post-burn-in block ends at arm $a$. This is the quantity that enters the visitation argument. Because the optimistic scores are recomputed after every reward observation inside the block, a one-step statement about the proposal kernel alone is not enough for the online algorithm. The theorem therefore uses these block-end probabilities directly.
Theorem \ref{thm:ralex} establishes that optimistic bias is compatible with canonical exploration, but it does not specify a regime where this bias provably enhances performance. To characterize these gains, we analyze the perpetual exploration process $\{\widetilde X_t\}_{t \ge 0}$, which follows the RALEX kernel indefinitely without transitioning to navigation.Fix a deterministic burn-in time $t_{\mathrm{gain}} \in \{0, \dots, T\}$ and define the mixing parameters:
\begin{equation}
\label{eq:tau_1}
\tau_1 \triangleq 1+W+\left\lceil \frac{1}{\rho\gamma}\log \frac{2}{(\pi_*)^{3/2}} \right\rceil; \quad \vartheta_1 \triangleq \frac{1}{2} \varepsilon_0^{\tau_1},
\end{equation}
where $\pi_* \triangleq \min_{a \in A} \pi(a)$. For each block index $j \ge 1$, let $s_j \triangleq t_{\mathrm{gain}} + (j-1)\tau_1$ and define the filtration $\widetilde{\mathcal G}_{j-1} \triangleq \sigma(\widetilde{\mathcal F}_{s_j})$. The quantity $\Pr(\widetilde X_{s_j+\tau_1} = a \mid \widetilde{\mathcal G}_{j-1})$ represents the probability that the $j$-th post-burn-in block terminates at arm $a$. Because optimistic scores are recomputed after every reward observation within a block, a one-step bound on the proposal kernel is insufficient for the online algorithm. Consequently, our analysis utilizes these block-end probabilities directly to quantify the visitation benefits of the reward-aware bias.
\begin{theorem}[RALEX performance gain]
\label{thm:ralex_gain}
Assume the conditions of Theorem \ref{thm:ralex} and let $\{\widetilde{X}_t\}$ be the perpetual exploration process. Fix a subset $B \subseteq A$ containing $a^\star$ and constants $\kappa_{\mathrm{gain}}(a) \ge \vartheta_1 \pi(a)$ for $a \in B$. Suppose that for every block index $j \ge 1$ with $s_j + \tau_1 \le T$:
\begin{equation}
\label{eq:gain_condition}
\Pr(\widetilde{X}_{s_j+\tau_1} = a \mid \widetilde{\mathcal{G}}_{j-1}) \ge \kappa_{\mathrm{gain}}(a), \quad \forall a \in B.
\end{equation}
Define $\kappa(a) \triangleq \kappa_{\mathrm{gain}}(a)$ for $a \in B$ and $\kappa(a) \triangleq \vartheta_1 \pi(a)$ otherwise. Let the complexity be:
\begin{equation}
\label{eq:psi_ra_gain}
\Psi_{\mathrm{RA}}^{\mathrm{gain}}(\delta) \triangleq \max \left\{ \frac{\log(4nT/\delta)}{\kappa(a^\star)\Delta_{\min}^2}, \max_{a \neq a^\star} \frac{\log(4nT/\delta)}{\kappa(a)\Delta(a)^2} \right\}.
\end{equation}
For a constant $C>0$, if $T \ge T_{\mathrm{nom}} \triangleq t_{\mathrm{gain}} + \tau_1 \lceil C \Psi_{\mathrm{RA}}^{\mathrm{gain}}(\delta) \rceil$, then with probability at least $1-\delta$: $\hat a^\star = a^\star$, $\sigma \le T_{\mathrm{nom}}$, and $R_{\mathrm{learn}}(T) \le T_{\mathrm{nom}}$.
\end{theorem}

The theorem is formulated using block-end probabilities, as these are the fundamental objects in our analysis. Appendix \ref{app:ralex-gain-proof} provides a sufficient condition for verifying this gain. Let $\mathcal C_j$ be the event that the first $\tau_1-1$ rounds of block $j$ use the canonical component of RALEX. If the final RALEX step in that block satisfies

\begin{equation}
\label{eq:gain_sufficient}
\begin{split}
\mathbb{E} \left[ \widetilde{K}_{s_j+\tau_1}(\widetilde{X}_{s_j+\tau_1-1},a) \;\middle|\; \widetilde{\mathcal{G}}_{j-1},\mathcal{C}_j \right] 
&\ge \left(\frac{\varepsilon_0}{2}+\beta(a)\right)\pi(a) \\
&\text{for some } \beta(a)\ge 0,
\end{split}
\end{equation}
then the block-end coefficient can be taken to be
\(
\kappa_{\mathrm{gain}}(a) =
\left(1+ 2\beta(a)/\varepsilon_0\right)\vartheta_1\pi(a).
\)
In other words, the improvement factor is determined by the amount of extra endpoint mass that the last reward-aware step contributes above the canonical floor.

To interpret when this yields a genuine reduction of the \emph{global} learning term, define
\[
g(a):=
\begin{cases}
\Delta_{\min} & \text{if } a=a^\star,\\
\Delta(a) & \text{if } a\neq a^\star,
\end{cases}
\qquad
\Lambda(a):=\frac{1}{\pi(a)g(a)^2},
\]
Let $\Lambda_{\max}:=\max_{a\in A}\Lambda(a)$. Thus the leading learning term in the safe RALEX bound is proportional to
\(
\vartheta_1^{-1}\log(4nT/\delta)\,\Lambda_{\max}.
\)
For a factor $\zeta>0$, define the $\zeta$-near-bottleneck set
\[
\mathcal H_\zeta
:=
\left\{
a\in A:\;
\Lambda(a)\ge \frac{\Lambda_{\max}}{1+\zeta}
\right\}.
\]
These are the arms whose safe contribution lies within a factor $1+\zeta$ of the worst one. We want to control this set as improving only the exact maximizers need not reduce the global maximum, because a near-bottleneck arm left unimproved can become the new dominant term after the gain is applied.

\begin{corollary}[Uniform gain on all near-bottleneck arms]
\label{cor:ralex_gain}
Assume the conditions of Theorem~\ref{thm:ralex_gain}, fix $\zeta>0$, and suppose that $\mathcal H_\zeta\subseteq B$ and
\(
\kappa_{\mathrm{gain}}(a)\ge (1+\zeta)\vartheta_1\pi(a)
\text{ for every }a\in \mathcal H_\zeta.
\)
Then
\[
\Psi_{\mathrm{RA}}^{\mathrm{gain}}(\delta)
\le
\frac{\log(4nT/\delta)}{(1+\zeta)\vartheta_1}\,\Lambda_{\max}.
\]
Consequently,
\begin{equation}
\label{eq:gain_bound}
T_{\mathrm{nom}}
=
O\!\left(
t_{\mathrm{gain}}
+
\frac{\tau_1\log(4nT/\delta)}{(1+\zeta)\vartheta_1}\,\Lambda_{\max}
\right).
\end{equation}
Relative to the safety theorem for RALEX, the leading learning term decreases by the factor $(1+\zeta)^{-1}$.
\end{corollary}

This is the strongest multiplicative improvement that Theorem~\ref{thm:ralex_gain} certifies without additional structure on the reward-aware proposal: the gain must cover every arm that could otherwise remain, or become, the dominant term.

\section{Numerical Simulations}
\label{sec:experiments}
We empirically ground our theoretical guarantees with a series of simulations designed to evaluate the practical performance of the proposed algorithm family. Our experiments, summarized in Figure~\ref{fig:quantitative_analysis}, focus on isolating the spatial behavior of different exploration strategies and quantifying their efficiency in identifying optimal regions. Specifically, we evaluate (a) mean cumulative regret on the Hard instance and (b) the distribution of adaptive stopping times across both Easy and Hard instances.

% \begin{wrapfigure}{r}{0.55\textwidth}
%     \centering
%     \vspace{-15pt}
%     \includegraphics[width=\linewidth]{fig-rlc.pdf} 
%     \caption{Quantitative performance.}
%     \label{fig:quantitative_analysis}
%     \vspace{-10pt}
% \end{wrapfigure}

\begin{figure}[h]
    \centering
    \includegraphics[width=\columnwidth]{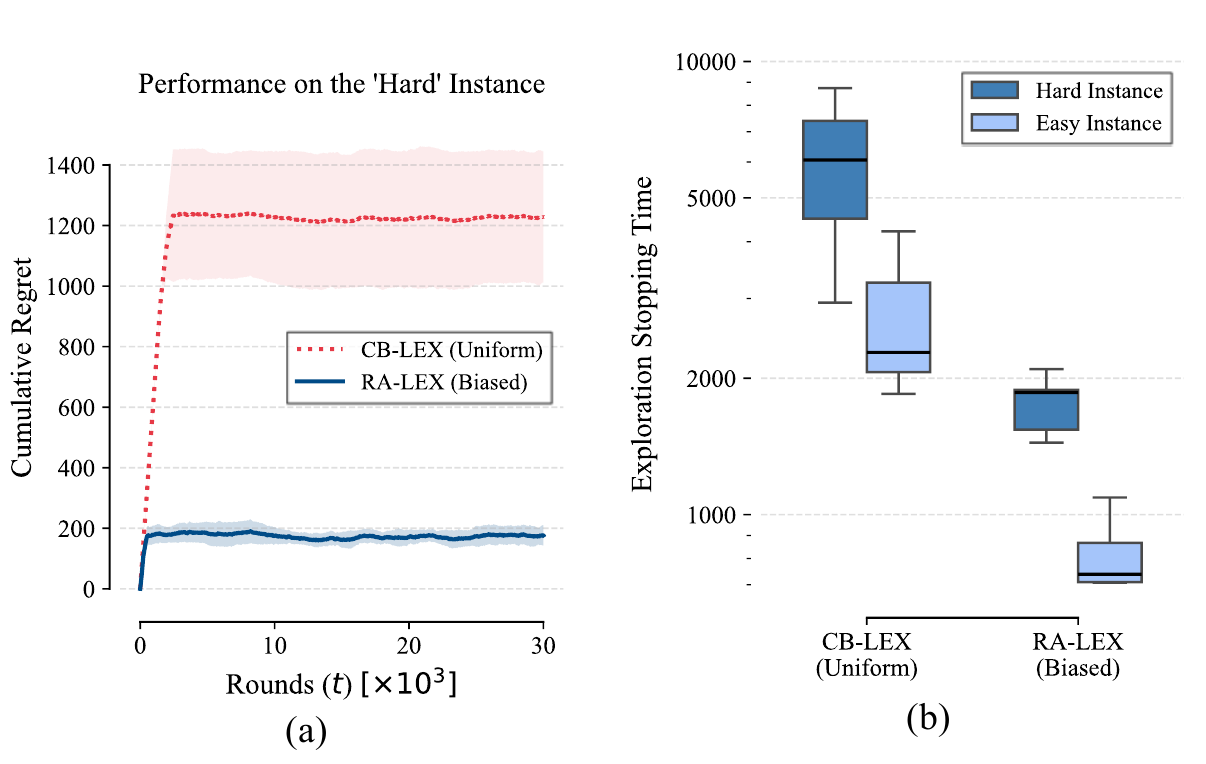} 
    \caption{Quantitative performance.}
    \label{fig:quantitative_analysis}
\end{figure}
\paragraph{Environment and Reward Model.}
The environment is modeled as a dynamic graph $G_t=(A,E_t)$ defined over a discrete set of $n=|A|=205$ nodes, simulated over a total horizon of $T = 70,000$ rounds. The network topology evolves via degree-preserving edge switches, maintaining a fixed degree sequence across all snapshots and thus a common stationary distribution $\pi$ as required by Definition~\ref{def:mixing-sequence}. Upon visiting a node $a \in A$, the agent observes a stochastic reward $r_t(a) \in [0,1]$ characterized by an underlying mean $\mu(a)$ perturbed by bounded, zero-mean noise.

\paragraph{Experimental Design.}
We evaluate performance on two distinct reward instances. The \textit{hard instance} features a small suboptimality gap, with a single optimal node $\mu(a^\star)=0.95$ and several competing arms in the range $[0.45, 0.65]$, requiring significant evidence to distinguish the true optimum. Conversely, the \textit{easy instance} features a large gap where suboptimal arms have low rewards in the range $[0.1, 0.2]$, allowing for rapid identification. For the dynamic graph, we set the edge appearance probability to 0.1 and the disappearance probability to 0.01. For the learning algorithms, we set the reward-bias to $\lambda=10$ for \textsc{RA-LEX}. Results are averaged over five random seeds, with shaded regions in regret curves representing one standard deviation.

\subsection{Results and Analysis}

% \begin{wrapfigure}{l}{0.5\textwidth}
%     \vspace{-15pt}
%     \centering
%     \includegraphics[width=\linewidth]{fig_heatmap_comparison_RLC.pdf} 
%     \caption{Exploration footprints. \textsc{CB-LEX} sweeps uniformly, while \textsc{RA-LEX} acts as a ``searchlight'' on the optimal region.}
%     \label{fig:heatmaps}
%     \vspace{-10pt}
% \end{wrapfigure}

\begin{figure}[h]
    \centering
    \includegraphics[width=\columnwidth]{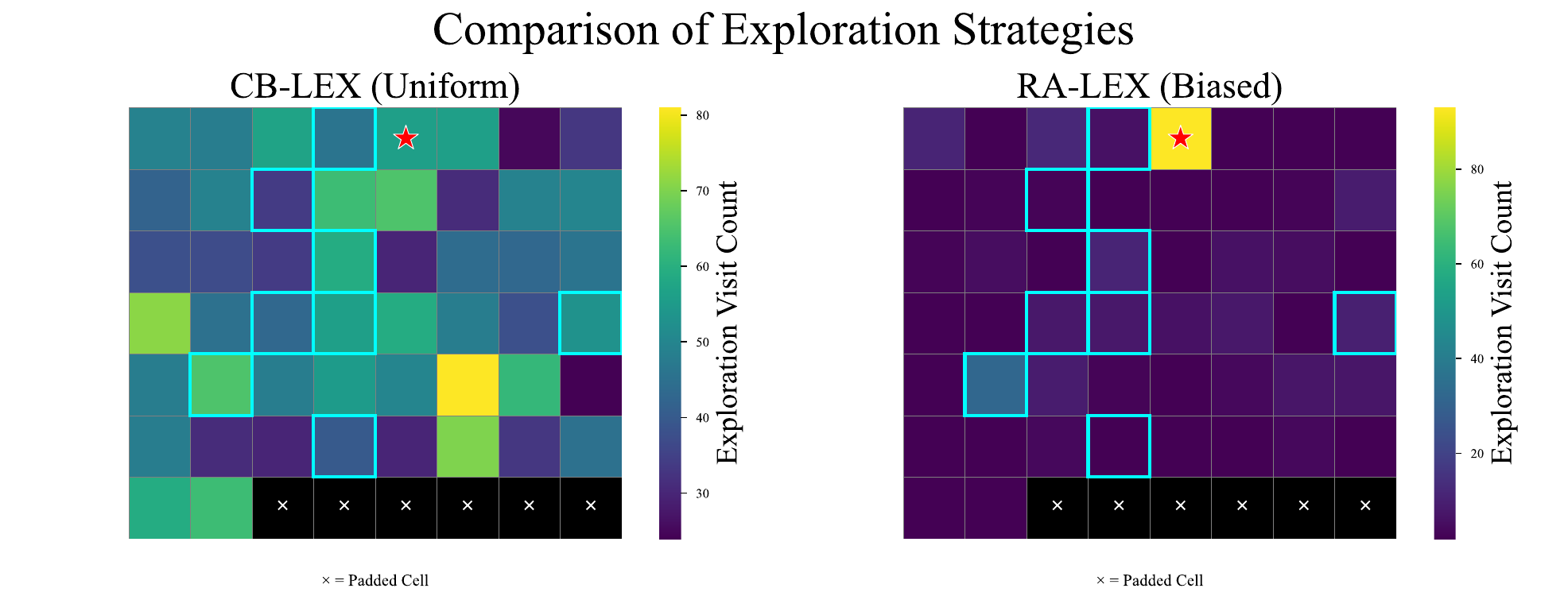}
    \caption{Exploration footprints. \textsc{CB-LEX} sweeps uniformly, while \textsc{RA-LEX} acts as a ``searchlight'' on the optimal region.}
    \label{fig:heatmaps}
\end{figure}

\textbf{Visualizing Exploration Strategies:} Figure~\ref{fig:heatmaps} provides the visual intuition behind the performance differences by comparing exploration-phase visitation patterns. The heatmap for \textsc{CB-LEX} is relatively uniform, showing a ``blind'' search of the entire node set. In stark contrast, the \textsc{RA-LEX} heatmap is highly concentrated, resembling a ``searchlight'' focused on the true high-reward zones. This visually confirms that its biased random walk successfully prioritizes promising regions, explaining its superior sample efficiency.

\textbf{Efficiency of Reward-Aware Exploration:} As shown in Figure~\ref{fig:quantitative_analysis}a, both algorithms achieve sublinear regret, indicated by the eventual flattening of their regret curves. However, the reward-aware \textsc{RA-LEX} achieves significantly lower final regret compared to the uniform exploration of \textsc{CB-LEX}, corroborating the benefit of biasing the exploration walk toward high-reward regions.

\textbf{Validation of the Adaptive Mechanism:} Figure~\ref{fig:quantitative_analysis}b confirms that the adaptive stopping rule correctly responds to statistical difficulty, with both algorithms stopping significantly earlier on the Easy instance. On the Hard instance, efficiency gains are pronounced: \textsc{RA-LEX} stops at a median of 1,850 rounds with a remarkably tight interquartile range (IQR) of 1,600 to 1,900. In contrast, \textsc{CB-LEX} requires a median of 6,000 rounds and exhibits a much wider variance (IQR of 4,500 to 7,500). This confirms that reward-aware biasing accelerates best-arm identification by nearly 70\% and makes the required exploration duration significantly more predictable.

\textbf{Robustness:} Further sensitivity analysis identifies a ``sweet spot'' for the bias parameter $\lambda$ that successfully balances exploiting high-reward regions with maintaining sufficient stochasticity to escape local optima.

% \input{Conclusion/c1}

% \printbibliography[heading=centered, title={\textsc{References}}]
% \end{refsection}

% % \printbibliography

% \clearpage

% \begin{refsection}
% \input{new_app}
% \printbibliography[heading=centered, title={\textsc{Appendix References}}]
% \end{refsection}

 \section{Conclusion}
\label{sec:conclusion}

% We studied stochastic bandits with local movement on dynamic graphs. The central point is that after exploration stops, the learner still has to move through the remaining graph sequence to reach the arm it wants to play. 

% Because exploration can end at a random time, conditions stated only over the whole horizon are not enough. One must control what happens on the suffix that starts when exploration stops. We addressed this by putting the structural assumption on the equal-neighbor walk itself. Under a condition that keeps this walk stable over time and guarantees enough useful rounds in every time window, we proved sublinear expected regret for three explore-then-commit algorithms. LEX gives the baseline argument, CB-LEX removes the need to tune the exploration length in advance, and RALEX shows that reward-aware movement can be added without losing worst-case safety, while also admitting a separate gain theorem when the bias helps.

% There are two natural next steps. One is to weaken the structural condition while keeping the stopping-time argument intact. The other is to extend the analysis to broader local kernels, such as Metropolis-type walks, which may enlarge the positive regime.
We investigated stochastic multi-armed bandits under local movement constraints on dynamic graphs, a setting where the structural necessity of post-exploration navigation decouples identification from exploitation. Because exploration concludes at an uncertain stopping time, whole-horizon connectivity is an insufficient metric for learnability; instead, structural properties must be controlled over the subsequent random suffix. This requirement led to our process-agnostic approach of placing assumptions directly on the graph's intrinsic walk.By imposing a mixing condition on the canonical walk, we established sublinear expected regret for a family of explore-then-commit algorithms. Our framework includes a fixed-budget baseline (LEX), an adaptive variant (CB-LEX) that eliminates oracle dependence, and a reward-aware strategy (RALEX) that maintains worst-case safety while providing provable performance gains.Future research will focus on weakening these structural requirements and extending the analysis to broader local kernels, such as Metropolis-type walks, to further expand the positive regime for bandit learning on dynamic networks

% One combined bibliography for the main paper and appendix
\printbibliography[heading=centered, title={\textsc{References}}]

\clearpage

\section{Related work and positioning}
\label{sec:related-work}

This paper lies at the intersection of stochastic bandits, graph-constrained decision problems, and random walks on dynamic graphs. The closest works come from several different literatures, and part of the contribution here is to separate them cleanly.

In the classical form of the stochastic bandits (\cite{auer2002finite,lattimore-szepesvari2020,slivkins2019introduction}), the learner repeatedly interacts with a fixed set of actions each of which yields an unknown real-valued payoff, or \textit{reward}. The objective of the learner is to maximize cumulative rewards by iteratively choosing actions that balance the discovery of new information (\textit{exploration}) with the pursuit of known high-payoff options (\textit{exploitation}). This model has been successfully applied across a wide range of applications, including clinical trials \cite{gittins1}, wireless communication \cite{cogradio}, high-dimensional decision and co-ordination problems (\cite{Kleinberg-lipschitz, Kleinberg-metric,chakraborty2026multi}) recommendation systems (\cite{Li_2010, bouneffouf:hal-00753401}), financial optimization \cite{brochu}, and economically motivated incentivized explorations (\cite{fraz, wang, liu20,chakraborty2024incentivized,chakraborty2025incentivized,sellke2021price}).  

Our confidence-based stopping rules borrow directly from that tradition. What changes in the present setting is that the learner cannot query an arbitrary arm at each round. The action it can take next is constrained by its current location in a changing graph, and after identifying the best arm it still has to reach it. This is why the proofs naturally separate into coverage, estimation, and navigation rather than following a single-phase regret analysis.

A first nearby literature studies graph-structured feedback, where playing one arm reveals losses or rewards of neighboring arms in a feedback graph (\cite{mannor-shamir2011,alon2017feedback,carpentier-valko2016,wen2024contextualgraphfeedback,valko2016bandits}). Those models use graphs to encode side observations. Our problem is different. The graph here determines feasibility of motion rather than observability, and the learner still receives only bandit feedback from the arm it actually visits. This makes local reachability and hitting times part of the statistical problem.

A second nearby literature studies bandits on a fixed graph with local movement constraints. The static graph-bandit model of Zhang, Johansson, and Li \cite{zhang-johansson-li2023graphbandit} and the multi-agent extension of Paschalidis, Zhang, and Li \cite{paschalidis-zhang-li2024multigraphbandit} assume a known graph and analyze regret on that fixed structure. Those papers are the right comparison point for sanity checks such as the complete graph, and they are one reason we explicitly position the present paper as a structural learnability paper rather than as a near-optimal regret paper. Our setting differs in two ways. The graph changes over time and is only locally observed, and the theorem is process-agnostic rather than specialized to a fixed known topology.

The dynamic random-walk literature is the main mathematical backdrop for the structural condition. Avin, Kouck\'y, and Lotker showed that the simple random walk on evolving graphs can behave very differently from the static case~\cite{avin-koucky-lotker2008}. Olshevsky and Tsitsiklis identified fixed-degree structure as a key positive regime for equal-neighbor dynamics, while even mild degree fluctuations can lead to exponentially slower convergence~\cite{olshevsky-tsitsiklis2013}. Sauerwald and Zanetti developed general mixing and hitting results for dynamic random walks with a common stationary distribution~\cite{sauerwald-zanetti2019}, and Shimizu and Shiraga sharpened this picture for reversible common-stationary walks, including strong bounds for lazy Metropolis on arbitrary dynamic graphs~\cite{shimizu-shiraga2022}. Cai, Sauerwald, and Zanetti studied randomly evolving graphs from the random-walk viewpoint~\cite{cai-sauerwald-zanetti2020}. Our structural condition is designed to align with this body of work. The main message is not that dynamic graphs are always learnable under the canonical walk. It is that once the canonical walk has a common stationary distribution and recurring absolute-gap contraction in every time window, a clean bandit theorem becomes possible.

This also explains our design choice of keeping the canonical walk instead of moving immediately to a Metropolis-type kernel. Metropolis walks can enlarge the positive regime because they enforce a common stationary distribution by construction~\cite{shimizu-shiraga2022}. We do not use them in this first paper because the question we want to answer is more basic: when does the graph-induced canonical walk already suffice? That question leads naturally to a characterization theorem. A theorem for Metropolis on arbitrary dynamic graphs would be broader, but it would answer a different structural question.

The most closely related dynamic-graph bandit papers are recent and model-specific. One line formulates bandit learning on dynamic graphs through block-density and temporal-stability conditions and studies LEX-, CB-LEX-, and RALEX-type algorithms on that basis~\cite{chakraborty-et-al2025dynamicgraphbandits}. Another line analyzes i.i.d.\ Erd\H{o}s--R\'enyi and edge-Markovian graph processes and proves process-specific upper and lower bounds~\cite{fmab2026}. The present paper is not a superset of those results. It is orthogonal to them. Relative to the block-density formulation, the main conceptual change here is that the structural condition is stated on the walk itself, not on a whole-horizon count of graph blocks. This makes the condition shift-invariant and therefore robust to the stopping times created by confidence-based exploration and by the navigation suffix. Relative to the process-specific analyses, the present paper trades model coverage for a walk-centric process-agnostic theorem in the common-stationary regime. The same shift also explains the algorithmic changes: the confidence-based rule is corrected from a top-two comparison to a leader-versus-all-competitors certificate, and the reward-aware analysis is split into a worst-case safety theorem and a separate gain theorem.

Finally, the problem is related in spirit to bandits with switching or movement costs (\cite{dekel-ding-koren-peres2014,koren-livni-mansour2017metric,koren-livni-mansour2017adaptive}). In those models all arms remain globally available and movement enters as an additional penalty. Here movement is part of feasibility itself. The learner may know the best arm and still be unable to play it for many rounds because the graph does not currently provide a route. That distinction is exactly why the navigation term appears as a separate structural object in our regret decomposition.

Our paper can be understood as a structural characterization for locally constrained stochastic bandits on dynamic graphs in the common-stationary regime. Its novelty is not a stronger regret guarantee on every graph family. Its is the walk-centric, shift-invariant sufficient condition, the corrected confidence-based stopping rule, the clean decomposition into learning and navigation, and the extension of that program from a canonical baseline to a reward-aware exploration rule that is safe in the worst case and provably advantageous under an additional gain condition.

\section{Why weaker conditions fail}
\label{app:weaker-conditions-fail}

This appendix records two negative examples that motivate Definition~\ref{def:mixing-sequence}. The first shows that connectivity of the horizon-wide union graph does not prevent linear regret. The second shows that even a positive fraction of individually good time blocks is not enough if those blocks are allowed to occur only at the end of the horizon.

\subsection{Horizon-wide connectivity is insufficient}
\label{app:horizon-long}

\begin{proposition}[Horizon-wide connectivity does not prevent linear regret]
\label{prop:horizon-long}
There exist $n\ge 2$, $T\ge 2$, a reward instance with a unique optimal arm $a^\star$ and $\Delta_{\min}=1$, and a graph sequence $\{G_t\}_{t=1}^T$ such that the union graph
\[
G_{[1,T]}:=\bigcup_{t=1}^T G_t
\]
is the complete graph $K_n$, but every local policy started from $a_0\sim\mathrm{Unif}(A)$ satisfies
\[
\mathbb E[R(T)]\ge \Bigl(1-\frac1n\Bigr)(T-1).
\]
\end{proposition}

\begin{proof}
Let $A_1:=\{a^\star\}$ and $A_2:=A\setminus\{a^\star\}$. For rounds $t=1,\dots,T-1$, let $G_t$ be the disjoint union of the isolated vertex $a^\star$ and the complete graph on $A_2$. At the final round let $G_T=K_n$.

The union graph over the whole horizon is $K_n$. If the learner starts at any vertex in $A_2$, which happens with probability $1-1/n$, then it cannot reach $a^\star$ during the first $T-1$ rounds because $a^\star$ is isolated throughout that period. Since $\Delta_{\min}=1$, the learner incurs regret at least $1$ on each of those rounds. Therefore
\[
\mathbb E[R(T)]
\ge
\Bigl(1-\frac1n\Bigr)(T-1).
\]
\end{proof}

\subsection{A whole-horizon fraction of good blocks is insufficient}
\label{app:global-fraction-insufficient}

\begin{proposition}[A whole-horizon density of good blocks does not imply learnability]
\label{prop:global-density-insufficient}
Fix any block length $\alpha\in\mathbb N$ and any fraction $\nu\in(0,1)$. There exist arbitrarily large horizons $T$, a reward instance with a unique optimal arm $a^\star$ and $\Delta_{\min}=1$, and a graph sequence $\{G_t\}_{t=1}^T$ with the following properties. When the horizon is partitioned into $\alpha$-blocks, at least a $\nu$-fraction of those blocks consist entirely of complete graphs. Consecutive graphs differ by at most one edge. Nevertheless, every local policy started from $a_0\sim \mathrm{Unif}(A)$ suffers expected regret $\Omega(T)$.
\end{proposition}

\begin{proof}
Let
\[
H:=K_{A\setminus\{a^\star\}}\cup \{a^\star\},
\]
the graph obtained by isolating $a^\star$ and keeping the remaining $n-1$ vertices as a clique. Choose a horizon $T$ such that
\[
m:=\lfloor T/\alpha\rfloor
\]
is arbitrarily large. Let
% \[
% L:=\left\lceil \frac{n-1}{\alpha}\right\rceil+1,
% \qquad
% m_{\mathrm{good}}:=\lceil \nu m\rceil,
% \qquad
% m_{\mathrm{bad}}:=m-m_{\mathrm{good}}-L.
% \]
% For all sufficiently large $m$, $m_{\mathrm{bad}}$ is positive.

\begin{gather*}
L \triangleq \lceil (n-1)/\alpha \rceil + 1, \quad m_{\mathrm{good}} \triangleq \lceil \nu m\rceil, \\
m_{\mathrm{bad}} \triangleq m - m_{\mathrm{good}} - L.
\end{gather*}
For all sufficiently large $m$, $m_{\mathrm{bad}}$ is positive.

During the first $m_{\mathrm{bad}}$ blocks, set every graph equal to $H$. During the next $L\alpha$ rounds, add the missing edges incident to $a^\star$ one at a time until the graph becomes complete, and after that keep it complete. During the final $m_{\mathrm{good}}$ blocks, set every graph equal to $K_n$. By construction, consecutive graphs differ by at most one edge, and at least a $\nu$-fraction of the $\alpha$-blocks consist entirely of complete graphs.

If the learner starts outside $a^\star$, which happens with probability $1-1/n$, then it cannot reach $a^\star$ during the first $m_{\mathrm{bad}}\alpha$ rounds because $a^\star$ is isolated throughout that whole prefix. Since $\Delta_{\min}=1$, the regret on each of those rounds is at least $1$. Therefore
\[
\mathbb E[R(T)]
\ge
\Bigl(1-\frac1n\Bigr)m_{\mathrm{bad}}\alpha.
\]
Because $L$ depends only on $n$ and $\alpha$,
\[
m_{\mathrm{bad}}
=
m-\lceil \nu m\rceil-L
\ge
(1-\nu)m-(L+1).
\]
Hence for all sufficiently large $m$,
\[
m_{\mathrm{bad}}\alpha
\ge
\frac{1-\nu}{2}\,m\alpha
=
\Omega(T).
\]
Thus a condition that only counts how many good blocks appear over the full horizon, even when consecutive graphs change only one edge at a time, still does not imply learnability.
\end{proof}

\section{The common-stationary sliding-window condition}
\label{app:mixing-condition}

Definition~\ref{def:mixing-sequence} is deliberately stated on the induced walk rather than on a union graph. The dynamic random-walk literature shows that once the stationary distribution is allowed to drift over time, even sequences of connected bounded-degree expanders can behave very differently from static graphs~\cite{sauerwald-zanetti2019,olshevsky-tsitsiklis2013,shimizu-shiraga2022}. The present paper therefore isolates a regime in which the canonical walk remains analyzable without changing the exploration kernel itself. A broader route would be to replace the canonical walk by a Metropolis-type kernel, but that would answer a different structural question.

The following proposition records the main graph-side subclass covered by Definition~\ref{def:mixing-sequence}.

\begin{proposition}[A fixed degree sequence gives a common stationary distribution]
\label{prop:fixed-degree-stationary}
Suppose there exists a function $d:A\to\{0,\dots,n-1\}$ such that
\[
\deg_t(a)=d(a)
\qquad\text{for every }a\in A\text{ and every }t\ge 1.
\]
Then every kernel $U_t$ in \eqref{eq:walk_matrix} is reversible with respect to
\[
\pi(a)=\frac{1+d(a)}{\sum_{b\in A}(1+d(b))}.
\]
In particular, if every snapshot is $d$-regular, then $\pi$ is uniform.
\end{proposition}

\begin{proof}
Fix $t$ and let
\[
Z:=\sum_{b\in A}(1+d(b)).
\]
For any $u,v\in A$ such that $v\in N_t(u)\cup\{u\}$, we have $u\in N_t(v)\cup\{v\}$ because the graph is undirected. Using the fixed degree sequence,
% \[
% \pi(u)U_t(u,v)
% =
% \frac{1+d(u)}{Z}\cdot \frac{1}{1+d(u)}
% =
% \frac{1}{Z}
% =
% \frac{1+d(v)}{Z}\cdot \frac{1}{1+d(v)}
% =
% \pi(v)U_t(v,u).
% \]

\begin{align*}
\pi(u)U_t(u,v) &= \frac{1+d(u)}{Z}\cdot \frac{1}{1+d(u)} = \frac{1}{Z} \\
&= \frac{1+d(v)}{Z}\cdot \frac{1}{1+d(v)} = \pi(v)U_t(v,u).
\end{align*}

If $v\notin N_t(u)\cup\{u\}$, then both sides are zero. Hence detailed balance holds for every pair $(u,v)$, so $\pi$ is stationary and $U_t$ is reversible with respect to $\pi$.
\end{proof}

The proposition should be read as a sufficient graph-side criterion, not as the only possible realization of Definition~\ref{def:mixing-sequence}. The theorem statements below are written in kernel language because the proofs use only reversibility, a common stationary distribution, and repeated one-step contraction.

\section{Technical lemmas}
\label{app:technical-lemmas}

We now prove the three ingredients used in the analysis of LEX: a uniform mixing lemma, a visitation lemma on a thinned trajectory, and a navigation lemma.

For a function $f:A\to\mathbb R$, define
\[
\langle f,g\rangle_\pi := \sum_{a\in A}\pi(a)f(a)g(a),
\qquad
\|f\|_{2,\pi}:=\sqrt{\langle f,f\rangle_\pi}.
\]
Let
\[
L_2^0(\pi):=\Bigl\{f:A\to\mathbb R:\sum_{a\in A}\pi(a)f(a)=0\Bigr\}.
\]
When a kernel $P$ acts on functions, we use the convention
\[
(Pf)(v):=\sum_{u\in A} P(v,u)f(u).
\]
Under reversibility this is the self-adjoint operator associated with the Markov chain on $L_2(\pi)$. We also write
\[
[x]_+ := \max\{x,0\}
\]
for the positive part of a real number.

\begin{lemma}[One-step contraction on a good round]
\label{lem:one-step-contraction}
Let $P$ be a reversible Markov kernel with stationary distribution $\pi$. If $\operatorname{gap}_{\mathrm{abs}}(P)\ge \gamma$, then
\[
\|Pf\|_{2,\pi}\le (1-\gamma)\|f\|_{2,\pi}
\qquad\text{for every }f\in L_2^0(\pi).
\]
\end{lemma}

\begin{proof}
Because $P$ is reversible, it is self-adjoint on $L_2(\pi)$. Let $\psi_1,\dots,\psi_n$ be an orthonormal eigenbasis with eigenvalues $\lambda_1(P),\dots,\lambda_n(P)$ and $\psi_1\equiv 1$. Every $f\in L_2^0(\pi)$ has an expansion
\[
f=\sum_{i=2}^n c_i\psi_i.
\]
Therefore
\[
Pf=\sum_{i=2}^n c_i\lambda_i(P)\psi_i
\]
and
\[
\|Pf\|_{2,\pi}^2
=
\sum_{i=2}^n c_i^2\lambda_i(P)^2
\le
\max_{2\le i\le n}|\lambda_i(P)|^2 \sum_{i=2}^n c_i^2.
\]
Since $\operatorname{gap}_{\mathrm{abs}}(P)\ge \gamma$, we have $\max_{2\le i\le n}|\lambda_i(P)|\le 1-\gamma$, so
\[
\|Pf\|_{2,\pi}^2 \le (1-\gamma)^2\|f\|_{2,\pi}^2.
\]
Taking square roots completes the proof.
\end{proof}

\begin{lemma}[Uniform mixing from any start time]
\label{lem:uniform-mixing}
Assume that $\{G_t\}_{t\ge 1}$ is $(W,\rho,\gamma,\pi)$-mixing for the canonical walk. Let
\[
P_{s,m}:=U_{s+m}\cdots U_{s+1}
\]
be the product of kernels over $m$ consecutive rounds started at time $s$. Then for every $s\ge 0$, every state $x\in A$, and every $m\ge W$,
\[
\bigl\|\delta_x P_{s,m}-\pi\bigr\|_{\mathrm{TV}}
\le
\frac{1}{2\sqrt{\pi_*}}\,e^{-\rho\gamma(m-W)}.
\]
Consequently, for every $\varepsilon\in(0,1)$,
\[
\tau_{\mathrm{mix}}(\varepsilon)
=
W+\left\lceil \frac{1}{\rho\gamma}\log\frac{1}{2\varepsilon\sqrt{\pi_*}}\right\rceil
\]
satisfies
\[
\bigl\|\delta_x P_{s,\tau_{\mathrm{mix}}(\varepsilon)}-\pi\bigr\|_{\mathrm{TV}}
\le \varepsilon
\qquad
\text{for all }s\ge 0\text{ and }x\in A.
\]
\end{lemma}

\begin{proof}
Fix $s\ge 0$ and an initial distribution $p^{(0)}$ at time $s$. For $j\ge 0$, let $p^{(j)}:=p^{(0)}U_{s+1}\cdots U_{s+j}$ be the distribution after $j$ steps, written as a row vector. Define the centered likelihood ratio
\[
f^{(j)}(v):=\frac{p^{(j)}(v)}{\pi(v)}-1.
\]
Since $\pi$ is stationary for every $U_t$ and detailed balance holds, we have for each $v\in A$,
\begin{align*}
f^{(j+1)}(v)
&=
\frac{1}{\pi(v)}\sum_{u\in A} p^{(j)}(u)U_{s+j+1}(u,v)-1 \\
&=
\frac{1}{\pi(v)}\sum_{u\in A} \pi(u)\bigl(1+f^{(j)}(u)\bigr)U_{s+j+1}(u,v)-1 \\
&=
\sum_{u\in A}\bigl(1+f^{(j)}(u)\bigr)\frac{\pi(u)U_{s+j+1}(u,v)}{\pi(v)}-1 \\
&=
\sum_{u\in A}\bigl(1+f^{(j)}(u)\bigr)U_{s+j+1}(v,u)-1 \\
&=
\sum_{u\in A} f^{(j)}(u)U_{s+j+1}(v,u)
=
(U_{s+j+1}f^{(j)})(v).
\end{align*}
Thus
\[
f^{(m)}=P_{s,m}f^{(0)}
\]
as functions.

Let $N_g(s,m)$ denote the number of indices in $\{s+1,\dots,s+m\}$ whose kernels satisfy $\operatorname{gap}_{\mathrm{abs}}(U_t)\ge \gamma$. Every interval of length $W$ contains at least $\lceil \rho W\rceil$ such indices. Partition the interval $\{s+1,\dots,s+m\}$ into $\lfloor m/W\rfloor$ disjoint blocks of length $W$ and one remainder block. This gives
\[
N_g(s,m)\ge \Bigl\lfloor \frac{m}{W}\Bigr\rfloor \lceil \rho W\rceil
\ge \rho(m-W).
\]
At every non-good round, the operator norm of $U_t$ on $L_2^0(\pi)$ is at most $1$, and at every good round Lemma~\ref{lem:one-step-contraction} gives a factor $(1-\gamma)$. Since $f^{(0)}\in L_2^0(\pi)$, repeated application yields
% \[
% \|f^{(m)}\|_{2,\pi}
% \le
% (1-\gamma)^{N_g(s,m)}\|f^{(0)}\|_{2,\pi}
% \le
% e^{-\gamma N_g(s,m)}\|f^{(0)}\|_{2,\pi}
% \le
% e^{-\rho\gamma(m-W)}\|f^{(0)}\|_{2,\pi}.
% \]
\begin{align*}
\|f^{(m)}\|_{2,\pi} &\le (1-\gamma)^{N_g(s,m)}\|f^{(0)}\|_{2,\pi} \\
&\le e^{-\gamma N_g(s,m)}\|f^{(0)}\|_{2,\pi} \\
&\le e^{-\rho\gamma(m-W)}\|f^{(0)}\|_{2,\pi}.
\end{align*}

Now specialize to the point mass $p^{(0)}=\delta_x$. Then
\[
f^{(0)}(u)=\frac{\mathbf 1\{u=x\}}{\pi(x)}-1.
\]
A direct calculation gives
\[
\|f^{(0)}\|_{2,\pi}^2
=
\sum_{u\in A}\pi(u)\left(\frac{\mathbf 1\{u=x\}}{\pi(x)}-1\right)^2
=
\frac{1}{\pi(x)}-1
\le
\frac{1}{\pi_*}.
\]
Finally,
\[
\|\delta_xP_{s,m}-\pi\|_{\mathrm{TV}}
=
\frac12\sum_{u\in A}\pi(u)|f^{(m)}(u)|
\le
\frac12\|f^{(m)}\|_{2,\pi}
\]
by Cauchy--Schwarz. Combining the previous displays gives
\[
\|\delta_xP_{s,m}-\pi\|_{\mathrm{TV}}
\le
\frac{1}{2\sqrt{\pi_*}}e^{-\rho\gamma(m-W)}.
\]
The claimed formula for $\tau_{\mathrm{mix}}(\varepsilon)$ is obtained by solving
\[
\frac{1}{2\sqrt{\pi_*}}e^{-\rho\gamma(m-W)}\le \varepsilon.
\]
\end{proof}

The next lemma is a one-sided Chernoff bound for adapted Bernoulli variables. We use it with indicators of the form $\mathbf 1\{X_{j\tau}=a\}$ on a thinned trajectory.

\begin{lemma}[A one-sided Chernoff bound for adapted Bernoulli variables]
\label{lem:adapted-chernoff}
Let $(\mathcal F_j)_{j\ge 0}$ be a filtration, and let $Y_1,\dots,Y_m$ be $\{0,1\}$-valued random variables such that $Y_j$ is $\mathcal F_j$-measurable and
\[
\mathbb E[Y_j\mid \mathcal F_{j-1}] \ge p
\qquad\text{almost surely for every }j.
\]
Then
\[
\Pr\!\left(\sum_{j=1}^m Y_j \le \frac{mp}{2}\right)
\le
\exp\!\left(-\frac{1-\log 2}{2}\,mp\right).
\]
\end{lemma}

\begin{proof}
Let $S_m:=\sum_{j=1}^m Y_j$. Since $Y_j\in\{0,1\}$,
\[
2^{-Y_j}=1-\frac{Y_j}{2}.
\]
Therefore
\[
\mathbb E[2^{-Y_j}\mid \mathcal F_{j-1}]
=
1-\frac{1}{2}\mathbb E[Y_j\mid \mathcal F_{j-1}]
\le
1-\frac{p}{2}
\le
e^{-p/2}.
\]
Iterating conditional expectation gives
\[
\mathbb E[2^{-S_m}]
=
\mathbb E\!\left[\prod_{j=1}^m 2^{-Y_j}\right]
\le
e^{-mp/2}.
\]
If $S_m\le mp/2$, then $2^{-S_m}\ge 2^{-mp/2}$. Markov's inequality therefore yields
% \[
% \Pr\!\left(S_m\le \frac{mp}{2}\right)
% \le
% 2^{mp/2}\,\mathbb E[2^{-S_m}]
% \le
% 2^{mp/2}e^{-mp/2}
% =
% \exp\!\left(-\frac{1-\log 2}{2}\,mp\right).
% \]
\begin{align*}
\Pr\left(S_m\le \frac{mp}{2}\right) &\le 2^{mp/2}\,\mathbb{E}[2^{-S_m}] \\
&\le 2^{mp/2}e^{-mp/2} = \exp\left(-\frac{1-\log 2}{2}\,mp\right).
\end{align*}
\end{proof}

We now use Lemma~\ref{lem:uniform-mixing} to obtain a visitation guarantee on a thinned subsequence of the trajectory.

\begin{lemma}[Uniform visitation on a thinned trajectory]
\label{lem:uniform-visitation}
Assume that $\{G_t\}_{t\ge 1}$ is $(W,\rho,\gamma,\pi)$-mixing for the canonical walk. Let
\[
\tau_0
:=
W+\left\lceil \frac{2}{\rho\gamma}\log\frac{1}{\pi_*}\right\rceil
\]
and let $X_t$ be the trajectory of the canonical walk. For any integer $m\ge 1$ and any arm $a\in A$, define
\[
N_m(a):=\sum_{j=1}^m \mathbf 1\{X_{j\tau_0}=a\}.
\]
There exists a universal constant $C_{\mathrm{vis}}>0$ such that if
\[
m \ge \frac{C_{\mathrm{vis}}}{\pi_*}\log\frac{n}{\delta},
\]
then with probability at least $1-\delta$,
\[
N_m(a)\ge \frac{m\pi(a)}{4}
\qquad\text{for every }a\in A.
\]
Consequently,
\[
\varphi_{m\tau_0}(a)\ge \frac{m\pi(a)}{4}
\qquad\text{for every }a\in A
\]
with probability at least $1-\delta$.
\end{lemma}

\begin{proof}
Fix an arm $a\in A$. Let
\[
\mathcal F_j:=\sigma(X_0,X_1,\dots,X_{j\tau_0})
\]
for $j\ge 0$, so that $\mathbf 1\{X_{j\tau_0}=a\}$ is $\mathcal F_j$-measurable. By Lemma~\ref{lem:uniform-mixing} with $\varepsilon=\pi_*/2$, the distribution of $X_{j\tau_0}$ conditional on $\mathcal F_{j-1}$ is within total variation distance $\pi_*/2$ of $\pi$. Therefore
\[
\Pr(X_{j\tau_0}=a\mid \mathcal F_{j-1})
\ge
\pi(a)-\frac{\pi_*}{2}
\ge
\frac{\pi(a)}{2}
\]
almost surely for every $j\ge 1$.

Applying Lemma~\ref{lem:adapted-chernoff} to
\[
Y_j:=\mathbf 1\{X_{j\tau_0}=a\}
\]
with $p=\pi(a)/2$ gives
\[
\Pr\!\left(N_m(a)\le \frac{m\pi(a)}{4}\right)
\le
\exp\!\left(
-\frac{1-\log 2}{4}\,m\pi(a)
\right).
\]
Since $\pi(a)\ge \pi_*$, the right-hand side is at most $\delta/n$ whenever
\[
m \ge \frac{4}{1-\log 2}\cdot \frac{1}{\pi_*}\log\frac{n}{\delta}.
\]
Thus the claim holds with any constant
\[
C_{\mathrm{vis}} \ge \frac{4}{1-\log 2}.
\]
Finally, $N_m(a)\le \varphi_{m\tau_0}(a)$ because every thinned visit is an actual visit. The same lower bound therefore holds for $\varphi_{m\tau_0}(a)$.
\end{proof}

The same mixing argument also gives a navigation bound.

\begin{lemma}[Navigation time]
\label{lem:navigation-time}
Assume that $\{G_t\}_{t\ge 1}$ is $(W,\rho,\gamma,\pi)$-mixing for the canonical walk, and let $\tau_0$ be as in Lemma~\ref{lem:uniform-visitation}. Consider the navigation phase of Algorithm~\ref{alg:template}, started at an arbitrary time and arbitrary current node, and suppose the target arm is $b\in A$. Then
\[
\mathbb E[\tau_{\mathrm{hit}}(b)] \le \frac{2\tau_0}{\pi(b)}.
\]
Consequently, if $b=a^\star$, then
\[
\mathbb E[R_{\mathrm{nav}}] \le \Delta_{\max}\,\mathbb E[\tau_{\mathrm{hit}}(a^\star)] \le \frac{2\Delta_{\max}\tau_0}{\pi(a^\star)}.
\]
\end{lemma}

\begin{proof}
Fix the target arm $b$. Let $s$ be the time at which the navigation phase starts, and let $X_t$ denote the actual navigation process of Algorithm~\ref{alg:template} for $t\ge s$. Thus $X_s$ is the current node at the start of navigation, and for each $t>s$,
\[
X_t=
\begin{cases}
b & \text{if } b\in L_t(X_{t-1}),\\[4pt]
\text{a sample from }U_t(X_{t-1},\cdot) & \text{otherwise,}
\end{cases}
\]
after which the process stays at $b$ forever.

Define also a pure canonical walk $Z_t$ started from the same state and time:
\[
Z_s:=X_s,
\qquad
Z_t\sim U_t(Z_{t-1},\cdot)
\quad\text{for }t>s.
\]

We couple $X$ and $Z$ on the same probability space as follows. As long as the actual process has not yet hit $b$, if at round $t$ we have $b\notin L_t(X_{t-1})$, then necessarily $X_{t-1}=Z_{t-1}$, and we sample a common transition
\[
Y_t\sim U_t(X_{t-1},\cdot)=U_t(Z_{t-1},\cdot)
\]
and set
\[
X_t=Z_t=Y_t.
\]
If instead $b\in L_t(X_{t-1})$, then the actual navigation rule sets $X_t=b$, while $Z_t$ is sampled from
\[
U_t(Z_{t-1},\cdot)=U_t(X_{t-1},\cdot).
\]
Once $X$ has hit $b$, we keep $X$ at $b$ forever and continue $Z$ as a pure canonical walk.

Under this coupling, the two processes are identical up to the first round on which the actual process hits $b$. At that round the actual process is already at $b$, whereas the pure walk may or may not be there. Therefore
\[
\tau_{\mathrm{hit}}^X(b)\le \tau_{\mathrm{hit}}^Z(b)
\qquad\text{almost surely,}
\]
where $\tau_{\mathrm{hit}}^X(b)$ and $\tau_{\mathrm{hit}}^Z(b)$ denote the hitting times of $b$ for $X$ and $Z$, respectively. It is thus enough to bound the hitting time of the pure canonical walk $Z$.

Let
\[
\mathcal F_t^Z:=\sigma(Z_s,Z_{s+1},\dots,Z_t)
\]
be the natural filtration of $Z$. For each block index $k\ge 1$, define the block start time
\[
u_k:=s+(k-1)\tau_0.
\]
Conditional on $\mathcal F_{u_k}^Z$, the future of $Z$ is the canonical walk started from the point mass at $Z_{u_k}$ and time $u_k$. By Lemma~\ref{lem:uniform-mixing} and the definition of $\tau_0$, the distribution of $Z_{u_k+\tau_0}$ is within total variation distance at most $\pi_*/2$ of $\pi$. Hence
\[
\Pr\!\left(Z_{u_k+\tau_0}=b \mid \mathcal F_{u_k}^Z\right)
\ge
\pi(b)-\frac{\pi_*}{2}
\ge
\frac{\pi(b)}{2}.
\]

Now define
\[
M:=\min\{k\ge 1: Z_{u_k+\tau_0}=b\}.
\]
The previous display shows that at the start of every block, conditional on the entire past, the probability that the block ends at $b$ is at least $\pi(b)/2$. Therefore
\[
\Pr(M\ge m+1)\le \left(1-\frac{\pi(b)}{2}\right)^m
\qquad\text{for every }m\ge 0,
\]
so $M$ is stochastically dominated by a geometric random variable with success parameter $\pi(b)/2$. In particular,
\[
\mathbb E[M]\le \frac{2}{\pi(b)}.
\]

If $Z_{u_M+\tau_0}=b$, then the pure walk has certainly hit $b$ by time $u_M+\tau_0=s+M\tau_0$. Hence
\[
\tau_{\mathrm{hit}}^Z(b)\le M\tau_0
\qquad\text{almost surely,}
\]
and therefore
\[
\mathbb E[\tau_{\mathrm{hit}}^Z(b)]
\le
\tau_0\,\mathbb E[M]
\le
\frac{2\tau_0}{\pi(b)}.
\]
Using $\tau_{\mathrm{hit}}^X(b)\le \tau_{\mathrm{hit}}^Z(b)$ almost surely gives
\[
\mathbb E[\tau_{\mathrm{hit}}(b)]
=
\mathbb E[\tau_{\mathrm{hit}}^X(b)]
\le
\frac{2\tau_0}{\pi(b)}.
\]

If $b=a^\star$, then every round before the hit incurs expected regret at most $\Delta_{\max}$, so
\[
\mathbb E[R_{\mathrm{nav}}]
\le
\Delta_{\max}\,\mathbb E[\tau_{\mathrm{hit}}(a^\star)]
\le
\frac{2\Delta_{\max}\tau_0}{\pi(a^\star)}.
\]
\end{proof}

\section{Proofs of the main results for LEX}
\label{app:lex-proof}

We now combine the previous lemmas to prove the fixed-budget guarantee for LEX and the resulting expected regret bound.

\begin{proof}[Proof of Theorem~\ref{thm:lex_high_prob}]
Let
\[
T_{\exp}=m\tau_0\le T
\]
with
\[
m \ge C\,\Psi_{\mathrm{LEX}}(\delta),
\]
where $C$ is a sufficiently large universal constant to be chosen below. Since LEX uses the canonical walk during exploration, its trajectory during the first $T_{\exp}$ rounds is exactly the trajectory of $\{U_t\}$.

Define the visitation event
\[
\mathcal E_{\mathrm{vis}}
:=
\left\{
\varphi_{T_{\exp}}(a)\ge \frac{m\pi(a)}{4}\ \text{for every }a\in A
\right\}.
\]
By Lemma~\ref{lem:uniform-visitation}, there exists a universal constant $C_{\mathrm{vis}}$ such that
\[
\Pr(\mathcal E_{\mathrm{vis}})\ge 1-\frac{\delta}{2}
\]
whenever
\[
m \ge \frac{C_{\mathrm{vis}}}{\pi_*}\log\frac{2n}{\delta}.
\]
This condition is implied by the displayed lower bound on $m$ after increasing $C$ if necessary. Indeed, define
\[
g(a):=
\begin{cases}
\Delta_{\min} & \text{if } a=a^\star,\\
\Delta(a) & \text{if } a\neq a^\star.
\end{cases}
\]
Since $g(a)\le 1$ for every arm, we have
\[
\frac{1}{\pi(a)g(a)^2}\ge \frac{1}{\pi(a)}.
\]
Taking the maximum over all arms and using the definition of $\Psi_{\mathrm{LEX}}(\delta)$ gives
\[
\Psi_{\mathrm{LEX}}(\delta)\ge \frac{\log(2n/\delta)}{\pi_*}.
\]

We next control the empirical means. Because the exploration kernel of LEX does not depend on rewards, the random counts $\varphi_{T_{\exp}}(a)$ are determined entirely by the graph sequence and the walk randomness. Conditional on the event $\mathcal E_{\mathrm{vis}}$, arm $a$ has been sampled at least
\[
k_a := \frac{m\pi(a)}{4}
\]
times. Conditional on any fixed value of $\varphi_{T_{\exp}}(a)\ge k_a$, the samples from arm $a$ are independent draws from $\mathcal D(a)$, so Hoeffding's inequality gives
\[
\Pr\!\left(
|\hat\mu_{T_{\exp}}(a)-\mu(a)|\ge \eta_a
\ \middle|\
\varphi_{T_{\exp}}(a)\ge k_a
\right)
\le
2\exp\!\left(-2k_a\eta_a^2\right),
\]
where we choose
\[
\eta_a:=
\begin{cases}
\Delta_{\min}/2 & \text{if } a=a^\star,\\
\Delta(a)/2 & \text{if } a\neq a^\star.
\end{cases}
\]
With this choice,
\[
2\exp\!\left(-2k_a\eta_a^2\right)
=
2\exp\!\left(-\frac{m\pi(a)g(a)^2}{8}\right).
\]
Hence, if
\[
m \ge \frac{8}{\pi(a)g(a)^2}\log\frac{4n}{\delta},
\]
then the right-hand side is at most $\delta/(2n)$. By the lower bound on $m$, this holds simultaneously for every arm once $C$ is chosen large enough. Let
% \[
% \mathcal E_{\mathrm{conc}}
% :=
% \left\{
% |\hat\mu_{T_{\exp}}(a^\star)-\mu(a^\star)|<\frac{\Delta_{\min}}{2}
% \right\}
% \cap
% \left\{
% |\hat\mu_{T_{\exp}}(a)-\mu(a)|<\frac{\Delta(a)}{2}
% \ \text{for every }a\neq a^\star
% \right\}.
% \]
\begin{align*}
\mathcal{E}_{\mathrm{conc}} &\triangleq \left\{ |\hat{\mu}_{T_{\exp}}(a^\star) - \mu(a^\star)| < \frac{\Delta_{\min}}{2} \right\} \\
&\quad \cap \bigcap_{a \neq a^\star} \left\{ |\hat{\mu}_{T_{\exp}}(a) - \mu(a)| < \frac{\Delta(a)}{2} \right\}.
\end{align*}
A union bound over all $n$ arms yields
\[
\Pr(\mathcal E_{\mathrm{conc}}\mid \mathcal E_{\mathrm{vis}})\ge 1-\frac{\delta}{2}.
\]

On the event $\mathcal E_{\mathrm{vis}}\cap \mathcal E_{\mathrm{conc}}$, for every $a\neq a^\star$ we have
\[
\hat\mu_{T_{\exp}}(a^\star)
>
\mu(a^\star)-\frac{\Delta_{\min}}{2}
\ge
\mu(a^\star)-\frac{\Delta(a)}{2}
=
\mu(a)+\frac{\Delta(a)}{2}
>
\hat\mu_{T_{\exp}}(a).
\]
Hence $\hat a^\star=a^\star$ on this event. Using
\[
\Pr(\mathcal E_{\mathrm{vis}}\cap \mathcal E_{\mathrm{conc}})
=
\Pr(\mathcal E_{\mathrm{vis}})\Pr(\mathcal E_{\mathrm{conc}}\mid \mathcal E_{\mathrm{vis}})
\ge
\left(1-\frac{\delta}{2}\right)^2
\ge
1-\delta,
\]
we obtain
\[
\Pr(\hat a^\star=a^\star)\ge 1-\delta.
\]

Finally, the exploration regret is bounded deterministically by the exploration length:
\[
R_{\mathrm{learn}}(T)\le T_{\exp}.
\]
Substituting the chosen value of $T_{\exp}$ gives the displayed upper bound.
\end{proof}

\begin{proof}[Proof of Lemma~\ref{lem:navigation-regret}]
The lemma is exactly the special case $b=a^\star$ of Lemma~\ref{lem:navigation-time}.
\end{proof}

\begin{proof}[Proof of Corollary~\ref{cor:lex_expected}]
Choose $\delta=T^{-2}$ and define
\[
T_{\exp}^{\mathrm{nom}}:=\tau_0\left\lceil C\,\Psi_{\mathrm{LEX}}(T^{-2})\right\rceil.
\]
The algorithm uses the actual exploration budget
\[
T_{\exp}:=\min\{T,T_{\exp}^{\mathrm{nom}}\}.
\]

If $T<T_{\exp}^{\mathrm{nom}}$, then the learner never uses more than $T$ exploration rounds and the regret is bounded deterministically by
\[
R(T)\le T\le T_{\exp}^{\mathrm{nom}}.
\]
This already matches the claimed order of the learning term.

It remains to consider the case $T_{\exp}^{\mathrm{nom}}\le T$. In that case Theorem~\ref{thm:lex_high_prob} applies with budget $T_{\exp}=T_{\exp}^{\mathrm{nom}}$. Let
\[
\mathcal E:=\{\hat a^\star=a^\star\}.
\]
The theorem gives
\[
\Pr(\mathcal E)\ge 1-\frac{1}{T^2}.
\]
By Lemma~\ref{lem:navigation-regret}, on $\mathcal E$ the total regret is at most
\[
R(T)\le T_{\exp}^{\mathrm{nom}}+R_{\mathrm{nav}},
\]
and therefore
\[
\mathbb E[R(T)\mid \mathcal E]
\le
T_{\exp}^{\mathrm{nom}}+\frac{2\Delta_{\max}\tau_0}{\pi(a^\star)}.
\]
On the complement $\mathcal E^c$, the regret is always at most $T$ because rewards lie in $[0,1]$. Hence
\begin{align*}
\mathbb E[R(T)]
&=
\mathbb E[R(T)\mid \mathcal E]\Pr(\mathcal E)
+
\mathbb E[R(T)\mid \mathcal E^c]\Pr(\mathcal E^c) \\
&\le
T_{\exp}^{\mathrm{nom}}+\frac{2\Delta_{\max}\tau_0}{\pi(a^\star)} + T\cdot \frac{1}{T^2}.
\end{align*}
The final term is at most $1$. Since
\[
\Psi_{\mathrm{LEX}}(T^{-2})
=
\max\left\{
\frac{\log(2nT^2)}{\pi(a^\star)\Delta_{\min}^2},
\max_{a\neq a^\star}\frac{\log(2nT^2)}{\pi(a)\Delta(a)^2}
\right\},
\]
and $\log(2nT^2)=O(\log(2nT))$, this yields the stated bound. Since $\tau_0$ and $\pi$ are independent of $T$, the right-hand side is $O(\log T)$ plus a $T$-independent navigation term, and therefore it is $o(T)$.
\end{proof}

\section{Proofs of the main results for CB-LEX}
\label{app:cblex-proof}

This appendix proves the confidence-based result using the same structural lemmas that were already used for LEX. The new difficulty is that the exploration length is now a stopping time. The proof resolves this by comparing CB-LEX to a perpetual exploration walk that ignores the stopping rule and continues to follow the canonical kernel for the entire horizon. The actual algorithm and the perpetual walk are identical until the stopping time, so it is enough to prove that the stopping condition must have become true by a deterministic time on the perpetual walk.

\begin{proof}[Proof of Theorem~\ref{thm:cblex}]
Let
\[
m_\star
:=
\left\lceil C\,\Psi_{\mathrm{CB}}(\delta)\right\rceil
\qquad\text{and}\qquad
t_\star:=m_\star\tau_0.
\]
By assumption, $t_\star\le T$.

We first define the perpetual exploration process. Let $\widetilde X_0:=a_0$, and for every $t=1,\dots,T$ let
\[
\widetilde X_t \sim U_t(\widetilde X_{t-1},\cdot),
\qquad
\widetilde r_t\sim \mathcal D(\widetilde X_t).
\]
Thus $\{\widetilde X_t\}_{t=0}^T$ is the canonical walk run for the full horizon, regardless of whether the stopping rule has triggered. From this trajectory we define the corresponding counts, empirical means, and confidence bounds:
\[
\widetilde\varphi_t(a)
:=
\sum_{s=1}^t \mathbf 1\{\widetilde X_s=a\},
\]
\[
\widetilde{\hat\mu}_t(a)
:=
\begin{cases}
\dfrac{1}{\widetilde\varphi_t(a)}
\displaystyle\sum_{s\le t:\,\widetilde X_s=a}\widetilde r_s
& \text{if }\widetilde\varphi_t(a)\ge 1,\\[10pt]
0 & \text{if }\widetilde\varphi_t(a)=0,
\end{cases}
\]
\[
\widetilde w_t(a)
:=
\begin{cases}
1 & \text{if }\widetilde\varphi_t(a)=0,\\[6pt]
\sqrt{\dfrac{\log(4nT/\delta)}{2\widetilde\varphi_t(a)}}
& \text{if }\widetilde\varphi_t(a)\ge 1,
\end{cases}
\]
and
\[
\widetilde{\operatorname{LCB}}_t(a)
:=
\widetilde{\hat\mu}_t(a)-\widetilde w_t(a),
\qquad
\widetilde{\operatorname{UCB}}_t(a)
:=
\widetilde{\hat\mu}_t(a)+\widetilde w_t(a).
\]
Let $\widetilde b_t\in\arg\max_{a\in A}\widetilde{\hat\mu}_t(a)$ and let $\widetilde\sigma$ be the first time $t\le T$ such that
\[
\widetilde{\operatorname{LCB}}_t(\widetilde b_t)
\ge
\max_{a\neq \widetilde b_t}\widetilde{\operatorname{UCB}}_t(a),
\]
with the convention $\widetilde\sigma:=T$ if this event never occurs before the horizon.

The actual CB-LEX trajectory coincides with the perpetual exploration trajectory up to the actual stopping time $\sigma$, because both processes use the same kernel $U_t$ and the same observed rewards as long as exploration continues. Consequently, it is enough to prove that the perpetual walk satisfies $\widetilde\sigma\le t_\star$ and stops with output $a^\star$. Indeed, if the actual algorithm has not stopped before time $t_\star$, then up to time $t_\star$ it is identical to the perpetual walk, so the stopping certificate that holds for the perpetual walk by time $t_\star$ also holds for the actual algorithm. Therefore the actual stopping time also satisfies $\sigma\le t_\star$. We will therefore work with the perpetual walk from now on.

The first event is a uniform concentration bound for all empirical means up to time $T$:
\[
\mathcal E_{\mathrm{conf}}
:=
\left\{
\forall t\in\{1,\dots,T\},\ \forall a\in A:
\left|\widetilde{\hat\mu}_t(a)-\mu(a)\right|
\le \widetilde w_t(a)
\right\}.
\]
Fix an arm $a$ and a time $t$. If $\widetilde\varphi_t(a)=0$, then $\widetilde w_t(a)=1$ and
\[
\left|\widetilde{\hat\mu}_t(a)-\mu(a)\right|
=
|\mu(a)|
\le 1
=
\widetilde w_t(a),
\]
so the event holds automatically. Now condition on $\widetilde\varphi_t(a)=k\ge 1$. Conditional on this event, the $k$ rewards used in $\widetilde{\hat\mu}_t(a)$ are i.i.d.\ samples from $\mathcal D(a)$, and Hoeffding's inequality yields
% \[
% \Pr\!\left(
% \left|\widetilde{\hat\mu}_t(a)-\mu(a)\right|
% >
% \sqrt{\frac{\log(4nT/\delta)}{2k}}
% \ \middle|\
% \widetilde\varphi_t(a)=k
% \right)
% \le
% 2\exp\!\left(-2k\cdot \frac{\log(4nT/\delta)}{2k}\right)
% =
% \frac{\delta}{2nT}.
% \]
\begin{align*}
\Pr &\left( \left|\widetilde{\hat{\mu}}_t(a) - \mu(a)\right| > \sqrt{\frac{\log(4nT/\delta)}{2k}} \;\middle|\; \widetilde{\varphi}_t(a) = k \right) \\
&\quad \le 2\exp\left( -2k \cdot \frac{\log(4nT/\delta)}{2k} \right) = \frac{\delta}{2nT}.
\end{align*}
Since the bound does not depend on $k$, the same estimate holds unconditionally:
\[
\Pr\!\left(
\left|\widetilde{\hat\mu}_t(a)-\mu(a)\right|
>
\widetilde w_t(a)
\right)
\le
\frac{\delta}{2nT}.
\]
A union bound over all $a\in A$ and all $t\in\{1,\dots,T\}$ gives
\[
\Pr(\mathcal E_{\mathrm{conf}})\ge 1-\frac{\delta}{2}.
\]

The second event is a visitation bound at the deterministic time $t_\star$:
\[
\mathcal E_{\mathrm{vis}}
:=
\left\{
\widetilde\varphi_{t_\star}(a)\ge \frac{m_\star\pi(a)}{4}
\ \text{for every }a\in A
\right\}.
\]
The perpetual exploration trajectory is exactly the canonical walk, so Lemma~\ref{lem:uniform-visitation} applies. We claim that, after increasing the universal constant $C$ if necessary,
\[
\Pr(\mathcal E_{\mathrm{vis}})\ge 1-\frac{\delta}{2}.
\]
Indeed, Lemma~\ref{lem:uniform-visitation} guarantees this as soon as
\[
m_\star \ge \frac{C_{\mathrm{vis}}}{\pi_*}\log\frac{2n}{\delta},
\]
where $C_{\mathrm{vis}}$ is the universal constant from that lemma. Since every gap is at most $1$, we have
\[
\Psi_{\mathrm{CB}}(\delta)
\ge
\max\left\{
\frac{\log(4nT/\delta)}{\pi(a^\star)},
\max_{a\neq a^\star}\frac{\log(4nT/\delta)}{\pi(a)}
\right\}
\ge
\frac{\log(2n/\delta)}{\pi_*}.
\]
Thus the displayed lower bound on $m_\star$ follows from the definition of $m_\star$ once $C$ is chosen large enough.

We now show that on the event $\mathcal E_{\mathrm{conf}}\cap \mathcal E_{\mathrm{vis}}$, the perpetual stopping rule must hold at time $t_\star$. Fix any arm $a\neq a^\star$. On $\mathcal E_{\mathrm{vis}}$,
\[
\widetilde\varphi_{t_\star}(a)\ge \frac{m_\star\pi(a)}{4},
\qquad
\widetilde\varphi_{t_\star}(a^\star)\ge \frac{m_\star\pi(a^\star)}{4}.
\]
Hence
\[
\widetilde w_{t_\star}(a)
\le
\sqrt{\frac{\log(4nT/\delta)}{2(m_\star\pi(a)/4)}}
=
\sqrt{\frac{2\log(4nT/\delta)}{m_\star\pi(a)}},
\]
and similarly
\[
\widetilde w_{t_\star}(a^\star)
\le
\sqrt{\frac{2\log(4nT/\delta)}{m_\star\pi(a^\star)}}.
\]
Since
\[
m_\star \ge C\,\frac{\log(4nT/\delta)}{\pi(a)\Delta(a)^2}
\]
for every $a\neq a^\star$, choosing $C\ge 32$ gives
\[
\widetilde w_{t_\star}(a)\le \frac{\Delta(a)}{4}.
\]
Likewise, since
\[
m_\star \ge C\,\frac{\log(4nT/\delta)}{\pi(a^\star)\Delta_{\min}^2},
\]
the same choice $C\ge 32$ yields
\[
\widetilde w_{t_\star}(a^\star)\le \frac{\Delta_{\min}}{4}\le \frac{\Delta(a)}{4}.
\]

Now use the confidence event. For the optimal arm,
\[
\widetilde{\operatorname{LCB}}_{t_\star}(a^\star)
=
\widetilde{\hat\mu}_{t_\star}(a^\star)-\widetilde w_{t_\star}(a^\star)
\ge
\mu(a^\star)-2\widetilde w_{t_\star}(a^\star).
\]
For the suboptimal arm $a$,
\[
\widetilde{\operatorname{UCB}}_{t_\star}(a)
=
\widetilde{\hat\mu}_{t_\star}(a)+\widetilde w_{t_\star}(a)
\le
\mu(a)+2\widetilde w_{t_\star}(a).
\]
Subtracting gives
\[
\widetilde{\operatorname{LCB}}_{t_\star}(a^\star)
-
\widetilde{\operatorname{UCB}}_{t_\star}(a)
\ge
\Delta(a)-2\widetilde w_{t_\star}(a^\star)-2\widetilde w_{t_\star}(a)
\ge
0.
\]
Since this holds for every $a\neq a^\star$, we have
\[
\widetilde{\operatorname{LCB}}_{t_\star}(a^\star)
\ge
\max_{a\neq a^\star}\widetilde{\operatorname{UCB}}_{t_\star}(a).
\]
Because every confidence width is strictly positive, we also have
\[
\widetilde{\operatorname{UCB}}_{t_\star}(a)
=
\widetilde{\hat\mu}_{t_\star}(a)+\widetilde w_{t_\star}(a)
>
\widetilde{\hat\mu}_{t_\star}(a)
\qquad\text{for every }a\in A.
\]
Therefore, for each $a\neq a^\star$,
\[
\widetilde{\hat\mu}_{t_\star}(a^\star)
\ge
\widetilde{\operatorname{LCB}}_{t_\star}(a^\star)
\ge
\widetilde{\operatorname{UCB}}_{t_\star}(a)
>
\widetilde{\hat\mu}_{t_\star}(a).
\]
Thus $a^\star$ is the unique empirical leader at time $t_\star$. In particular, if
\[
\widetilde b_{t_\star}\in \arg\max_{a\in A}\widetilde{\hat\mu}_{t_\star}(a),
\]
then necessarily $\widetilde b_{t_\star}=a^\star$, and the perpetual stopping condition holds at time $t_\star$. Hence
\[
\widetilde\sigma\le t_\star.
\]

We next show that every stopping arm is correct on $\mathcal E_{\mathrm{conf}}$. Suppose the perpetual process stops at some time $t$ with arm $\widetilde b_t$. Then
\[
\widetilde{\operatorname{LCB}}_t(\widetilde b_t)
\ge
\max_{a\neq \widetilde b_t}\widetilde{\operatorname{UCB}}_t(a).
\]
For every competitor $a\neq \widetilde b_t$, the confidence event implies
\[
\mu(\widetilde b_t)\ge \widetilde{\operatorname{LCB}}_t(\widetilde b_t)
\qquad\text{and}\qquad
\mu(a)\le \widetilde{\operatorname{UCB}}_t(a).
\]
Hence
\[
\mu(\widetilde b_t)\ge \mu(a)
\qquad\text{for every }a\neq \widetilde b_t.
\]
Because the optimal arm is unique, this forces $\widetilde b_t=a^\star$. In particular, on $\mathcal E_{\mathrm{conf}}\cap \mathcal E_{\mathrm{vis}}$ the perpetual process stops by time $t_\star$ and stops with the correct arm.

Finally we transfer this conclusion back to the actual algorithm. As observed at the beginning of the proof, the actual and perpetual trajectories are identical up to the actual stopping time. On the event $\mathcal E_{\mathrm{conf}}\cap \mathcal E_{\mathrm{vis}}$, the perpetual stopping time is at most $t_\star$, and at that time the stopping condition is satisfied with leader $a^\star$. Therefore the actual algorithm must also stop by time $t_\star$, and when it stops it must return $a^\star$. We have proved that
\[
\Pr\!\left(
\sigma\le t_\star
\text{ and }
\hat a^\star=a^\star
\right)
\ge
1-\delta.
\]

Since each exploration round incurs regret at most $1$,
\[
R_{\mathrm{learn}}(T)
=
\sum_{t=1}^{\sigma}\bigl(\mu(a^\star)-r_t\bigr)
\le \sigma
\le t_\star
=
T_{\mathrm{CB}}^{\mathrm{nom}}
\]
on the same event. This completes the proof.
\end{proof}

\begin{proof}[Proof of Corollary~\ref{cor:cblex_expected}]
Define
\[
T_{\mathrm{CB}}^{\mathrm{nom}}
:=
\tau_0\left\lceil C\,\Psi_{\mathrm{CB}}(T^{-2})\right\rceil.
\]

If $T<T_{\mathrm{CB}}^{\mathrm{nom}}$, then the regret is bounded deterministically by
\[
R(T)\le T\le T_{\mathrm{CB}}^{\mathrm{nom}},
\]
so the claimed estimate is immediate.

Assume now that $T_{\mathrm{CB}}^{\mathrm{nom}}\le T$. Apply Theorem~\ref{thm:cblex} with $\delta=T^{-2}$ and let
\[
\mathcal E
:=
\left\{
\sigma\le T_{\mathrm{CB}}^{\mathrm{nom}}
\text{ and }
\hat a^\star=a^\star
\right\}.
\]
Then
\[
\Pr(\mathcal E)\ge 1-\frac{1}{T^2}.
\]
On $\mathcal E$, the exploration phase lasts at most $T_{\mathrm{CB}}^{\mathrm{nom}}$ rounds and the committed arm is the true optimum. Lemma~\ref{lem:navigation-regret} therefore gives
\[
\mathbb E[R(T)\mid \mathcal E]
\le
T_{\mathrm{CB}}^{\mathrm{nom}}
+
\frac{2\Delta_{\max}\tau_0}{\pi(a^\star)}.
\]
On the complement $\mathcal E^c$, the regret is always at most $T$. Hence
\begin{align*}
\mathbb E[R(T)]
&=
\mathbb E[R(T)\mid \mathcal E]\Pr(\mathcal E)
+
\mathbb E[R(T)\mid \mathcal E^c]\Pr(\mathcal E^c)\\
&\le
T_{\mathrm{CB}}^{\mathrm{nom}}
+
\frac{2\Delta_{\max}\tau_0}{\pi(a^\star)}
+
T\cdot \frac{1}{T^2}.
\end{align*}
The last term is at most $1$, which yields
\[
\mathbb E[R(T)]
\le
T_{\mathrm{CB}}^{\mathrm{nom}}
+
\frac{2\Delta_{\max}\tau_0}{\pi(a^\star)}
+
1.
\]

To obtain the displayed asymptotic form, note that
\[
\Psi_{\mathrm{CB}}(T^{-2})
=
\max\left\{
\frac{\log(4nT^3)}{\pi(a^\star)\Delta_{\min}^2},
\max_{a\neq a^\star}\frac{\log(4nT^3)}{\pi(a)\Delta(a)^2}
\right\},
\]
and $\log(4nT^3)=O(\log(nT))$. Because $\tau_0$ and $\pi$ are independent of the horizon, the right-hand side is logarithmic in $T$ plus the $T$-independent navigation term, and is therefore $o(T)$.
\end{proof}

\section{Proofs of the main results for RALEX}
\label{app:ralex-proof}

This appendix extends the confidence-based analysis to the reward-aware exploration kernel. Two features of RALEX require new work. The first is statistical. Because the exploration kernel depends on the past rewards, the visit counts are themselves reward-dependent, so the conditioning argument used for CB-LEX is no longer valid. The second is structural. The perpetual RALEX walk no longer coincides with the canonical walk, and it need not preserve the common stationary distribution from Definition~\ref{def:mixing-sequence}. The proof therefore isolates two ingredients. It first proves a uniform confidence lemma that remains valid under adaptive sampling by representing the rewards through arm-wise reward tables. It then uses only the canonical floor from \eqref{eq:ralex-floor}. Inside each block of length $\tau_0$, there is a fixed probability that the process follows the canonical walk for the entire block, and this is enough to recover a thinned visitation bound.

\begin{lemma}[Uniform confidence under adaptive sampling]
\label{lem:adaptive-confidence}
Fix a horizon $T$ and a confidence level $\delta\in(0,1)$. For each arm $a\in A$, let
\[
Y_{a,1},Y_{a,2},\dots
\]
be i.i.d. random variables with law $\mathcal D(a)$, and assume that the collections corresponding to different arms are mutually independent and independent of the graph process and all algorithmic randomization. Consider any exploration rule that, at each round $t$, selects an arm $X_t$ using the sigma-field generated by the past observations, the current graph information available to the learner, and arbitrary auxiliary randomization independent of the reward tables. Let
\[
\varphi_t(a):=\sum_{s=1}^t \mathbf 1\{X_s=a\}
\]
be the number of times arm $a$ has been pulled by time $t$, and define
\[
\hat\mu_t(a)
:=
\begin{cases}
\dfrac{1}{\varphi_t(a)}\displaystyle\sum_{i=1}^{\varphi_t(a)} Y_{a,i} & \text{if }\varphi_t(a)\ge 1,\\[10pt]
0 & \text{if }\varphi_t(a)=0,
\end{cases}
\]
with widths
\[
w_t(a)
:=
\begin{cases}
1 & \text{if }\varphi_t(a)=0,\\[6pt]
\sqrt{\dfrac{\log(4nT/\delta)}{2\varphi_t(a)}} & \text{if }\varphi_t(a)\ge 1.
\end{cases}
\]
Then the event
\[
\mathcal E_{\mathrm{adapt}}
:=
\left\{
\forall t\in\{1,\dots,T\},\ \forall a\in A:
\left|\hat\mu_t(a)-\mu(a)\right|\le w_t(a)
\right\}
\]
satisfies
\[
\Pr(\mathcal E_{\mathrm{adapt}})\ge 1-\frac{\delta}{2}.
\]
\end{lemma}

\begin{proof}
For every arm $a\in A$ and every integer $k\in\{1,\dots,T\}$, define
\[
\bar Y_{a,k}:=\frac{1}{k}\sum_{i=1}^k Y_{a,i}.
\]
Hoeffding's inequality gives
% \[
% \Pr\!\left(
% \left|\bar Y_{a,k}-\mu(a)\right|>
% \sqrt{\frac{\log(4nT/\delta)}{2k}}
% \right)
% \le
% 2\exp\!\left(-2k\cdot \frac{\log(4nT/\delta)}{2k}\right)
% =
% \frac{\delta}{2nT}.
% \]
\begin{align*}
\Pr &\left( \left|\bar{Y}_{a,k} - \mu(a)\right| > \sqrt{\frac{\log(4 n T / \delta)}{2 k}} \right) \\
&\quad \le 2 \exp \left(-2 k \cdot \frac{\log (4 n T / \delta)}{2 k}\right) = \frac{\delta}{2 n T}.
\end{align*}
Taking a union bound over all $a\in A$ and all $k\in\{1,\dots,T\}$ shows that the event
\[
\mathcal E_{\mathrm{table}}
:=
\left\{
\forall a\in A,\ \forall k\in\{1,\dots,T\}:
\left|\bar Y_{a,k}-\mu(a)\right|
\le
\sqrt{\frac{\log(4nT/\delta)}{2k}}
\right\}
\]
has probability at least $1-\delta/2$.

We claim that $\mathcal E_{\mathrm{table}}\subseteq \mathcal E_{\mathrm{adapt}}$. Fix $t\le T$ and $a\in A$. If $\varphi_t(a)=0$, then by definition $\hat\mu_t(a)=0$ and $w_t(a)=1$, so
\[
\left|\hat\mu_t(a)-\mu(a)\right|=
\mu(a)
\le 1=w_t(a).
\]
If instead $\varphi_t(a)=k\ge 1$, then the reward observed on the $i$-th pull of arm $a$ is $Y_{a,i}$, and therefore
\[
\hat\mu_t(a)=\frac{1}{k}\sum_{i=1}^k Y_{a,i}=\bar Y_{a,k}.
\]
On $\mathcal E_{\mathrm{table}}$ this implies
\[
\left|\hat\mu_t(a)-\mu(a)\right|
=
\left|\bar Y_{a,k}-\mu(a)\right|
\le
\sqrt{\frac{\log(4nT/\delta)}{2k}}
=
w_t(a).
\]
Since $t$ and $a$ were arbitrary, $\mathcal E_{\mathrm{table}}\subseteq \mathcal E_{\mathrm{adapt}}$, and hence
\[
\Pr(\mathcal E_{\mathrm{adapt}})
\ge
\Pr(\mathcal E_{\mathrm{table}})
\ge
1-\frac{\delta}{2}.
\]
\end{proof}

\begin{lemma}[Visitation under a canonical floor]
\label{lem:ralex-visitation}
Assume that the realized graph sequence is $(W,\rho,\gamma,\pi)$-mixing for the canonical walk. Let $\{K_t\}_{t\ge 1}$ be any adapted sequence of local kernels such that for every round $t$, every current node $u\in A$, and every $v\in A$,
\[
K_t(u,v)\ge \varepsilon_0 U_t(u,v)
\]
almost surely, where $\varepsilon_0\in(0,1]$. Let $X_t$ be the resulting trajectory, define
\[
\tau_0
:=
W+\left\lceil \frac{2}{\rho\gamma}\log\frac{1}{\pi_*}\right\rceil,
\qquad
\vartheta_0:=\frac{1}{2}\,\varepsilon_0^{\tau_0},
\]
and for every integer $m\ge 1$ and every arm $a\in A$ set
\[
N_m^{K}(a):=\sum_{j=1}^m \mathbf 1\{X_{j\tau_0}=a\}.
\]
There exists a universal constant $C_{\mathrm{vis}}^{\mathrm{RA}}>0$ such that if
\[
m \ge \frac{C_{\mathrm{vis}}^{\mathrm{RA}}}{\vartheta_0\pi_*}\log\frac{n}{\delta},
\]
then with probability at least $1-\delta$,
\[
N_m^{K}(a)\ge \frac{m\vartheta_0\pi(a)}{2}
\qquad\text{for every }a\in A.
\]
Consequently,
\[
\varphi_{m\tau_0}(a)\ge \frac{m\vartheta_0\pi(a)}{2}
\qquad\text{for every }a\in A
\]
with probability at least $1-\delta$.
\end{lemma}

\begin{proof}
If $\varepsilon_0=1$, then $K_t=U_t$ almost surely for every $t$, the process coincides with the canonical walk, and the conclusion follows from Lemma~\ref{lem:uniform-visitation} with $\vartheta_0=\tfrac12$. We therefore assume in the rest of the proof that $\varepsilon_0<1$.

Fix an arm $a\in A$. Let $\mathcal F_j$ be the full history of the process up to time $j\tau_0$, including states, rewards, and all internal randomization used to define the kernels. For convenience write
\[
\nu:=(j-1)\tau_0
\]
for the start of block $j$.

We first prove a one-block lower bound. Condition on $\mathcal F_{j-1}$, so that the current state $X_\nu$ is fixed. Because
\[
K_{\nu+r}(u,v)\ge \varepsilon_0 U_{\nu+r}(u,v)
\qquad\text{for every }u,v\in A
\]
almost surely, there exists, for each step $r=1,\dots,\tau_0$, a residual kernel $R_{\nu+r}$ such that
\[
K_{\nu+r}(u,\cdot)
=
\varepsilon_0 U_{\nu+r}(u,\cdot)
+
(1-\varepsilon_0)R_{\nu+r}(u,\cdot)
; \quad \forall u\in A.
\]
Now enlarge the probability space inside the block by introducing independent Bernoulli variables
\[
B_{\nu+1},\dots,B_{\nu+\tau_0}
\sim \mathrm{Bernoulli}(\varepsilon_0),
\]
independent of the past and independent of the reward draws inside the block. Generate the next state at time $\nu+r$ by sampling from $U_{\nu+r}$ when $B_{\nu+r}=1$ and from $R_{\nu+r}$ when $B_{\nu+r}=0$. This construction has the correct marginal law $K_{\nu+r}$ at each step.

Let
\[
\mathcal B_j:=\{B_{\nu+1}=\cdots=B_{\nu+\tau_0}=1\}.
\]
On $\mathcal B_j$, every transition in the block is sampled from the canonical kernel. Even though the kernels are data-dependent, the canonical component itself depends only on the current state and the current graph, so conditioning on $\mathcal B_j$ forces the state trajectory on this block to follow the canonical walk from $X_\nu$ for exactly $\tau_0$ steps. Since the canonical walk is $(W,\rho,\gamma,\pi)$-mixing, Lemma~\ref{lem:uniform-mixing} with $\varepsilon=\pi_*/2$ gives
\[
\Pr\!\left(X_{\nu+\tau_0}=a\mid \mathcal F_{j-1},\mathcal B_j\right)
\ge
\pi(a)-\frac{\pi_*}{2}
\ge
\frac{\pi(a)}{2}.
\]
Because $\Pr(\mathcal B_j\mid \mathcal F_{j-1})=\varepsilon_0^{\tau_0}$, we conclude that
\[
\Pr\!\left(X_{j\tau_0}=a\mid \mathcal F_{j-1}\right)
\ge
\varepsilon_0^{\tau_0}\cdot \frac{\pi(a)}{2}
=
\vartheta_0\pi(a)
\]
almost surely for every $j\ge 1$.

We now apply Lemma~\ref{lem:adapted-chernoff} to the adapted Bernoulli variables
\[
Y_j:=\mathbf 1\{X_{j\tau_0}=a\},
\qquad j=1,\dots,m,
\]
with success parameter $p=\vartheta_0\pi(a)$. This yields
\[
\Pr\!\left(N_m^{K}(a)\le \frac{m\vartheta_0\pi(a)}{2}\right)
\le
\exp\!\left(-\frac{1-\log 2}{2}\,m\vartheta_0\pi(a)\right).
\]
Since $\pi(a)\ge \pi_*$, the right-hand side is at most $\delta/n$ whenever
\[
m
\ge
\frac{2}{1-\log 2}\cdot \frac{1}{\vartheta_0\pi_*}\log\frac{n}{\delta}.
\]
Thus the conclusion holds for every arm simultaneously after a union bound over $a\in A$. The final claim follows from the deterministic inequality $N_m^{K}(a)\le \varphi_{m\tau_0}(a)$.
\end{proof}

\begin{proof}[Proof of Theorem~\ref{thm:ralex}]
Let
\[
m_\star
:=
\left\lceil C\,\Psi_{\mathrm{RA}}(\delta)\right\rceil
\qquad\text{and}\qquad
t_\star:=m_\star\tau_0.
\]
By assumption, $t_\star\le T$.

We first define the perpetual reward-aware exploration process. Let $\widetilde X_0:=a_0$, and for every $t=1,\dots,T$ construct $\widetilde X_t$ recursively as follows. From the perpetual history up to time $t-1$, compute the widths $\widetilde w_{t-1}(a)$, the optimistic scores
\[
\widetilde\xi_{t-1}(a):=\widetilde{\hat\mu}_{t-1}(a)+\widetilde w_{t-1}(a),
\]
the local proposal $\widetilde S_t$ from \eqref{eq:ralex-softmax}, the mixture weight $\widetilde\varepsilon_t$, and the kernel
\[
\widetilde K_t
:=
\widetilde\varepsilon_t U_t+(1-\widetilde\varepsilon_t)\widetilde S_t.
\]
Then sample
\[
\widetilde X_t\sim \widetilde K_t(\widetilde X_{t-1},\cdot),
\qquad
\widetilde r_t\sim \mathcal D(\widetilde X_t).
\]
This process ignores the stopping rule and explores for the full horizon. Let $\widetilde\varphi_t(a)$, $\widetilde{\hat\mu}_t(a)$, $\widetilde w_t(a)$, $\widetilde{\operatorname{LCB}}_t(a)$, $\widetilde{\operatorname{UCB}}_t(a)$, and $\widetilde b_t$ be the associated counts, empirical means, widths, confidence bounds, and empirical leaders. Let $\widetilde\sigma$ be the first time at which the stopping rule \eqref{eq:cblex-stop} holds for the perpetual process, with the convention $\widetilde\sigma:=T$ if it never triggers.

The actual RALEX trajectory coincides with the perpetual one up to the actual stopping time $\sigma$, because both processes use the same reward-aware kernel computed from the same history as long as exploration continues. Therefore, if we show that the perpetual process must stop by time $t_\star$ and must stop with the correct arm, then the actual algorithm must also satisfy
\[
\sigma\le t_\star
\qquad\text{and}\qquad
\hat a^\star=a^\star.
\]
Indeed, if the actual algorithm has not stopped before time $t_\star$, then up to time $t_\star$ it is identical to the perpetual process, so the stopping certificate that holds for the perpetual process by time $t_\star$ also holds for the actual algorithm. We may therefore work entirely with the perpetual process.

The first event is a uniform confidence event for the perpetual reward-aware exploration process:
\[
\widetilde{\mathcal E}_{\mathrm{conf}}
:=
\left\{
\forall t\in\{1,\dots,T\},\ \forall a\in A:
\left|\widetilde{\hat\mu}_t(a)-\mu(a)\right|\le \widetilde w_t(a)
\right\}.
\]
Unlike CB-LEX, this event cannot be obtained by conditioning on the sample counts, because the counts themselves depend on past rewards through the kernel $\widetilde K_t$. Lemma~\ref{lem:adaptive-confidence} applies directly to the perpetual process and therefore gives
\[
\Pr\!\left(\widetilde{\mathcal E}_{\mathrm{conf}}\right)\ge 1-\frac{\delta}{2}.
\]

The second event is a visitation bound at the deterministic time $t_\star$:
\[
\widetilde{\mathcal E}_{\mathrm{vis}}
:=
\left\{
\widetilde\varphi_{t_\star}(a)\ge \frac{m_\star\vartheta_0\pi(a)}{2}
\ \text{for every }a\in A
\right\}.
\]
Because $\widetilde K_t\ge \varepsilon_0 U_t$ pointwise for every $t$, Lemma~\ref{lem:ralex-visitation} applies to the perpetual trajectory. We claim that, after increasing the universal constant $C$ if necessary,
\[
\Pr\!\left(\widetilde{\mathcal E}_{\mathrm{vis}}\right)\ge 1-\frac{\delta}{2}.
\]
Indeed, Lemma~\ref{lem:ralex-visitation} guarantees this as soon as
\[
m_\star \ge \frac{C_{\mathrm{vis}}^{\mathrm{RA}}}{\vartheta_0\pi_*}\log\frac{2n}{\delta}.
\]
Since every gap is at most $1$, we have
% \[
% \Psi_{\mathrm{RA}}(\delta)
% \ge
% \max\left\{
% \frac{\log(4nT/\delta)}{\vartheta_0\pi(a^\star)},
% \max_{a\neq a^\star}\frac{\log(4nT/\delta)}{\vartheta_0\pi(a)}
% \right\}
% \ge
% \frac{\log(2n/\delta)}{\vartheta_0\pi_*},
% \]
\begin{align*}
\Psi_{\mathrm{RA}}(\delta) &\ge \max \left\{ \frac{\log(4nT/\delta)}{\vartheta_0\pi(a^\star)}, \max_{a \neq a^\star} \frac{\log(4nT/\delta)}{\vartheta_0\pi(a)} \right\} \\
&\ge \frac{\log(2n/\delta)}{\vartheta_0\pi_*}.
\end{align*}

so the displayed lower bound on $m_\star$ follows from the definition of $m_\star$ once $C$ is sufficiently large.

We now show that on $\widetilde{\mathcal E}_{\mathrm{conf}}\cap \widetilde{\mathcal E}_{\mathrm{vis}}$, the perpetual stopping rule must hold at time $t_\star$. Fix any arm $a\neq a^\star$. On $\widetilde{\mathcal E}_{\mathrm{vis}}$,
\[
\widetilde\varphi_{t_\star}(a)\ge \frac{m_\star\vartheta_0\pi(a)}{2},
\qquad
\widetilde\varphi_{t_\star}(a^\star)\ge \frac{m_\star\vartheta_0\pi(a^\star)}{2}.
\]
Hence
\[
\widetilde w_{t_\star}(a)
\le
\sqrt{\frac{\log(4nT/\delta)}{m_\star\vartheta_0\pi(a)}},
\qquad
\widetilde w_{t_\star}(a^\star)
\le
\sqrt{\frac{\log(4nT/\delta)}{m_\star\vartheta_0\pi(a^\star)}}.
\]
Since
\[
m_\star
\ge
C\,\frac{\log(4nT/\delta)}{\vartheta_0\pi(a)\Delta(a)^2}
\qquad\text{for every }a\neq a^\star,
\]
choosing $C\ge 16$ gives
\[
\widetilde w_{t_\star}(a)\le \frac{\Delta(a)}{4}.
\]
Likewise,
\[
m_\star
\ge
C\,\frac{\log(4nT/\delta)}{\vartheta_0\pi(a^\star)\Delta_{\min}^2}
\]
and the same choice $C\ge 16$ yields
\[
\widetilde w_{t_\star}(a^\star)\le \frac{\Delta_{\min}}{4}\le \frac{\Delta(a)}{4}.
\]

On the confidence event, these inequalities imply the same separation as in the proof of Theorem~\ref{thm:cblex}. For the optimal arm,
\[
\widetilde{\operatorname{LCB}}_{t_\star}(a^\star)
=
\widetilde{\hat\mu}_{t_\star}(a^\star)-\widetilde w_{t_\star}(a^\star)
\ge
\mu(a^\star)-2\widetilde w_{t_\star}(a^\star).
\]
For the suboptimal arm $a$,
\[
\widetilde{\operatorname{UCB}}_{t_\star}(a)
=
\widetilde{\hat\mu}_{t_\star}(a)+\widetilde w_{t_\star}(a)
\le
\mu(a)+2\widetilde w_{t_\star}(a).
\]
Subtracting gives
\[
\widetilde{\operatorname{LCB}}_{t_\star}(a^\star)
-
\widetilde{\operatorname{UCB}}_{t_\star}(a)
\ge
\Delta(a)-2\widetilde w_{t_\star}(a^\star)-2\widetilde w_{t_\star}(a)
\ge 0.
\]
Thus
\[
\widetilde{\operatorname{LCB}}_{t_\star}(a^\star)
\ge
\max_{a\neq a^\star}\widetilde{\operatorname{UCB}}_{t_\star}(a).
\]
Because every confidence width is strictly positive, we also have
\[
\widetilde{\operatorname{UCB}}_{t_\star}(a)
=
\widetilde{\hat\mu}_{t_\star}(a)+\widetilde w_{t_\star}(a)
>
\widetilde{\hat\mu}_{t_\star}(a)
\qquad\text{for every }a\in A.
\]
Hence for every $a\neq a^\star$,
\[
\widetilde{\hat\mu}_{t_\star}(a^\star)
\ge
\widetilde{\operatorname{LCB}}_{t_\star}(a^\star)
\ge
\widetilde{\operatorname{UCB}}_{t_\star}(a)
>
\widetilde{\hat\mu}_{t_\star}(a).
\]
Therefore $a^\star$ is the unique empirical leader at time $t_\star$, so the perpetual stopping condition holds at time $t_\star$ and $\widetilde\sigma\le t_\star$.

As in the proof of Theorem~\ref{thm:cblex}, the confidence event alone implies that every arm returned by the perpetual stopping rule is correct. Indeed, if the perpetual process stops at time $t$ with arm $\widetilde b_t$, then
\[
\widetilde{\operatorname{LCB}}_t(\widetilde b_t)
\ge
\max_{a\neq \widetilde b_t}\widetilde{\operatorname{UCB}}_t(a).
\]
On $\widetilde{\mathcal E}_{\mathrm{conf}}$, this implies
\[
\mu(\widetilde b_t)\ge \mu(a)
\qquad\text{for every }a\neq \widetilde b_t,
\]
and uniqueness of the optimum forces $\widetilde b_t=a^\star$. We conclude that on $\widetilde{\mathcal E}_{\mathrm{conf}}\cap \widetilde{\mathcal E}_{\mathrm{vis}}$,
\[
\widetilde\sigma\le t_\star
\qquad\text{and}\qquad
\widetilde b_{\widetilde\sigma}=a^\star.
\]
Transferring this back to the actual algorithm gives
\[
\Pr\!\left(
\sigma\le t_\star
\text{ and }
\hat a^\star=a^\star
\right)
\ge 1-\delta.
\]
Since each exploration round incurs regret at most $1$,
\[
R_{\mathrm{learn}}(T)
=
\sum_{t=1}^{\sigma}\bigl(\mu(a^\star)-r_t\bigr)
\le \sigma
\le t_\star
=
T_{\mathrm{RA}}^{\mathrm{nom}}
\]
on the same event. This completes the proof.
\end{proof}

\begin{proof}[Proof of Corollary~\ref{cor:ralex_expected}]
Define
\[
T_{\mathrm{RA}}^{\mathrm{nom}}
:=
\tau_0\left\lceil C\,\Psi_{\mathrm{RA}}(T^{-2})\right\rceil.
\]

If $T<T_{\mathrm{RA}}^{\mathrm{nom}}$, then the regret is bounded deterministically by
\[
R(T)\le T\le T_{\mathrm{RA}}^{\mathrm{nom}},
\]
so the claimed estimate is immediate.

Assume now that $T_{\mathrm{RA}}^{\mathrm{nom}}\le T$. Apply Theorem~\ref{thm:ralex} with $\delta=T^{-2}$ and let
\[
\mathcal E
:=
\left\{
\sigma\le T_{\mathrm{RA}}^{\mathrm{nom}}
\text{ and }
\hat a^\star=a^\star
\right\}.
\]
Then
\[
\Pr(\mathcal E)\ge 1-\frac{1}{T^2}.
\]
On $\mathcal E$, the exploration phase lasts at most $T_{\mathrm{RA}}^{\mathrm{nom}}$ rounds and the committed arm is the true optimum. Lemma~\ref{lem:navigation-regret} therefore gives
\[
\mathbb E[R(T)\mid \mathcal E]
\le
T_{\mathrm{RA}}^{\mathrm{nom}}
+
\frac{2\Delta_{\max}\tau_0}{\pi(a^\star)}.
\]
On the complement $\mathcal E^c$, the regret is at most $T$. Hence
\begin{align*}
\mathbb E[R(T)]
&=
\mathbb E[R(T)\mid \mathcal E]\Pr(\mathcal E)
+
\mathbb E[R(T)\mid \mathcal E^c]\Pr(\mathcal E^c)\\
&\le
T_{\mathrm{RA}}^{\mathrm{nom}}
+
\frac{2\Delta_{\max}\tau_0}{\pi(a^\star)}
+
T\cdot \frac{1}{T^2}.
\end{align*}
The last term is at most $1$, which yields
\[
\mathbb E[R(T)]
\le
T_{\mathrm{RA}}^{\mathrm{nom}}
+
\frac{2\Delta_{\max}\tau_0}{\pi(a^\star)}
+1.
\]
The displayed asymptotic form follows from the definition of $\Psi_{\mathrm{RA}}(T^{-2})$ and the fact that $\log(4nT^3)=O(\log(nT))$.
\end{proof}

\begin{theorem}[Formal learnability under common-stationary sliding-window mixing]
\label{thm:formal-existence}
Let the realized graph sequence be $(W,\rho,\gamma,\pi)$-mixing for the canonical walk. Then the process is learnable.
\end{theorem}

\begin{proof}
Corollary~\ref{cor:lex_expected} shows that LEX satisfies
% \[
% \mathbb E[R(T)]
% =
% O\!\left(
% \tau_0
% \max\left\{
% \frac{\log(2nT)}{\pi(a^\star)\Delta_{\min}^2},
% \max_{a\neq a^\star}\frac{\log(2nT)}{\pi(a)\Delta(a)^2}
% \right\}
% \right)
% +
% O\!\left(\frac{\tau_0}{\pi(a^\star)}\right),
% \]
\begin{align*}
\mathbb{E}[R(T)] &= O\left( \tau_0 \max\left\{ \frac{\log(2nT)}{\pi(a^\star)\Delta_{\min}^2}, \max_{a\neq a^\star}\frac{\log(2nT)}{\pi(a)\Delta(a)^2} \right\} \right) \\
&\quad + O\left(\frac{\tau_0}{\pi(a^\star)}\right).
\end{align*}
which is $o(T)$. Hence there exists a local policy with sublinear expected regret, namely LEX.
\end{proof}

\section{When reward awareness yields a provable gain}
\label{app:ralex-gain-proof}

This appendix upgrades the RALEX analysis from a safety theorem to a genuine improvement theorem. The key point is that the online RALEX kernel is updated inside each block. Because of this within-block adaptation, the proof cannot be written only in terms of a proposal kernel evaluated against a fixed stationary distribution. The correct object is the conditional probability that a whole post-burn-in block ends at a given arm. The theorem in the main text is stated directly on this quantity.

We work with the perpetual RALEX exploration process introduced in Appendix~\ref{app:ralex-proof}. All objects carrying tildes below refer to that perpetual process. Fix a deterministic burn-in time $t_{\mathrm{gain}}\in\{0,\dots,T\}$ and define
\[
\tau_1
:=
1+\tau_{\mathrm{mix}}\!\left(\frac{\pi_*}{4}\right),
\qquad
\vartheta_1:=\frac{1}{2}\,\varepsilon_0^{\tau_1}.
\]
For each block index $j\ge 1$, write
\[
s_j:=t_{\mathrm{gain}}+(j-1)\tau_1,
\qquad
\widetilde{\mathcal G}_{j-1}:=\sigma\!\left(\widetilde{\mathcal F}_{s_j}\right).
\]

\begin{lemma}[Safe one-block endpoint bound]
\label{lem:ralex-safe-one-block}
Assume the conditions of Theorem~\ref{thm:ralex}. Then for every arm $a\in A$, every time $s\in\{0,\dots,T-\tau_1\}$, and every realization of $\widetilde{\mathcal F}_s$,
\[
\Pr\!\left(\widetilde X_{s+\tau_1}=a\mid \widetilde{\mathcal F}_s\right)
\ge
\vartheta_1\pi(a).
\]
\end{lemma}

\begin{proof}
If $\varepsilon_0=1$, then the perpetual RALEX process is exactly the canonical walk, $\vartheta_1=\tfrac12$, and Lemma~\ref{lem:uniform-mixing} with $\varepsilon=\pi_*/4$ gives
\[
\Pr\!\left(\widetilde X_{s+\tau_1}=a\mid \widetilde{\mathcal F}_s\right)
\ge
\pi(a)-\frac{\pi_*}{4}
\ge
\frac{\pi(a)}{2}
=
\vartheta_1\pi(a).
\]
We therefore assume in the rest of the proof that $\varepsilon_0<1$.

Fix $s$ and condition on $\widetilde{\mathcal F}_s$. For each round $r=s+1,\dots,s+\tau_1$, decompose the perpetual RALEX kernel as
\[
\widetilde K_r(u,\cdot)
=
\varepsilon_0 U_r(u,\cdot)
+
(1-\varepsilon_0)\widetilde R_r(u,\cdot),
\]
where
\[
\widetilde R_r(u,\cdot)
:=
\frac{\widetilde K_r(u,\cdot)-\varepsilon_0 U_r(u,\cdot)}{1-\varepsilon_0}
\]
is the residual kernel. Introduce independent Bernoulli variables
\[
B_{s+1},\dots,B_{s+\tau_1}\sim \mathrm{Bernoulli}(\varepsilon_0),
\]
independent of the past and of all reward draws inside the block, and generate the trajectory on these rounds by using the canonical kernel $U_r$ whenever $B_r=1$ and the residual kernel $\widetilde R_r$ whenever $B_r=0$. This coupling has the correct marginal law $\widetilde K_r$ at every round.

Let
\[
\mathcal D_s:=\{B_{s+1}=\cdots=B_{s+\tau_1}=1\}.
\]
Then
\[
\Pr(\mathcal D_s\mid \widetilde{\mathcal F}_s)=\varepsilon_0^{\tau_1}.
\]
On $\mathcal D_s$, the process follows the canonical walk for all $\tau_1$ rounds of the block. Since
\[
\tau_1-1=\tau_{\mathrm{mix}}\!\left(\frac{\pi_*}{4}\right),
\]
Lemma~\ref{lem:uniform-mixing} gives
\[
\left\|\mathcal L\!\left(\widetilde X_{s+\tau_1}\mid \widetilde{\mathcal F}_s,\mathcal D_s\right)-\pi\right\|_{\mathrm{TV}}
\le
\frac{\pi_*}{4}.
\]
Therefore
\[
\Pr\!\left(\widetilde X_{s+\tau_1}=a\mid \widetilde{\mathcal F}_s,\mathcal D_s\right)
\ge
\pi(a)-\frac{\pi_*}{4}
\ge
\frac{3}{4}\pi(a)
\ge
\frac{1}{2}\pi(a).
\]
Multiplying by $\Pr(\mathcal D_s\mid \widetilde{\mathcal F}_s)=\varepsilon_0^{\tau_1}$ yields
\[
\Pr\!\left(\widetilde X_{s+\tau_1}=a\mid \widetilde{\mathcal F}_s\right)
\ge
\frac{1}{2}\,\varepsilon_0^{\tau_1}\pi(a)
=
\vartheta_1\pi(a).
\]
This proves the claim.
\end{proof}

The next lemma shows how a one-step endpoint advantage after a canonical warm-up propagates to a block-end gain. The statement is written directly in terms of the actual conditional endpoint probability. This avoids treating the last-step RALEX kernel as if it were fixed at block start.

\begin{lemma}[Warm-up endpoint gain implies a block-end gain]
\label{lem:ralex-warmup-gain}
Assume the conditions of Theorem~\ref{thm:ralex}. Fix an arm $a\in A$ and a time $s\in\{0,\dots,T-\tau_1\}$. Let
\[
\mathcal C_s:=\{B_{s+1}=\cdots=B_{s+\tau_1-1}=1\}
\]
be the event that the first $\tau_1-1$ rounds of the block use the canonical component in the coupling above. Assume that for some $\beta(a)\ge 0$,
\[
\mathbb E\!\left[
\widetilde K_{s+\tau_1}(\widetilde X_{s+\tau_1-1},a)
\ \middle|\
\widetilde{\mathcal F}_s,\mathcal C_s
\right]
\ge
\left(\frac{\varepsilon_0}{2}+\beta(a)\right)\pi(a)
\]
almost surely. Then
\[
\Pr\!\left(\widetilde X_{s+\tau_1}=a\mid \widetilde{\mathcal F}_s\right)
\ge
\Bigl(\vartheta_1+\varepsilon_0^{\tau_1-1}\beta(a)\Bigr)\pi(a).
\]
\end{lemma}

\begin{proof}
Fix $s$ and condition on $\widetilde{\mathcal F}_s$. By construction,
\[
\Pr(\mathcal C_s\mid \widetilde{\mathcal F}_s)=\varepsilon_0^{\tau_1-1}.
\]
On $\mathcal C_s$, the first $\tau_1-1$ rounds of the block are canonical, while the last round is taken with the actual RALEX kernel. Therefore
\begin{align*}
\Pr\!\left(\widetilde X_{s+\tau_1}=a\mid \widetilde{\mathcal F}_s\right)
&\ge
\Pr(\mathcal C_s\mid \widetilde{\mathcal F}_s)
\,\mathbb E\!\left[
\widetilde K_{s+\tau_1}(\widetilde X_{s+\tau_1-1},a)
\ \middle|\
\widetilde{\mathcal F}_s,\mathcal C_s
\right]\\
&\ge
\varepsilon_0^{\tau_1-1}
\left(\frac{\varepsilon_0}{2}+\beta(a)\right)\pi(a)\\
&=
\left(\frac{1}{2}\,\varepsilon_0^{\tau_1}+\varepsilon_0^{\tau_1-1}\beta(a)\right)\pi(a)\\
&=
\Bigl(\vartheta_1+\varepsilon_0^{\tau_1-1}\beta(a)\Bigr)\pi(a).
\end{align*}
This is the desired bound.
\end{proof}

\begin{proof}[Proof of Theorem~\ref{thm:ralex_gain}]
Let
\[
m_\star:=\left\lceil C\,\Psi_{\mathrm{RA}}^{\mathrm{gain}}(\delta)\right\rceil
\qquad\text{and}\qquad
t_\star:=t_{\mathrm{gain}}+m_\star\tau_1.
\]
By assumption, $t_\star\le T$.

As in the proof of Theorem~\ref{thm:ralex}, it is enough to work with the perpetual RALEX exploration process. If the actual algorithm has not stopped before time $t_\star$, then up to time $t_\star$ it is identical to the perpetual process. Therefore it is enough to prove that the perpetual process must satisfy
\[
\widetilde\sigma\le t_\star
\qquad\text{and}\qquad
\widetilde b_{\widetilde\sigma}=a^\star
\]
with probability at least $1-\delta$.

The first ingredient is the adaptive confidence event
\[
\widetilde{\mathcal E}_{\mathrm{conf}}
:=
\left\{
\forall t\in\{1,\dots,T\},\ \forall a\in A:
\left|\widetilde{\hat\mu}_t(a)-\mu(a)\right|\le \widetilde w_t(a)
\right\}.
\]
Lemma~\ref{lem:adaptive-confidence} gives
\[
\Pr\!\left(\widetilde{\mathcal E}_{\mathrm{conf}}\right)\ge 1-\frac{\delta}{2}.
\]

For each arm $a\in A$ and each block index $j\in\{1,\dots,m_\star\}$, define the block-end indicator
\[
Y_{j,a}:=\mathbf 1\{\widetilde X_{s_j+\tau_1}=a\}.
\]
These variables are adapted to the block filtration $\{\widetilde{\mathcal G}_j\}_{j\ge 0}$.

Fix an arm $a\in A$. If $a\in B$, then the theorem assumption gives
\[
\Pr\!\left(Y_{j,a}=1\mid \widetilde{\mathcal G}_{j-1}\right)
\ge
\kappa_{\mathrm{gain}}(a)
=
\kappa(a)
\qquad\text{for every }j\in\{1,\dots,m_\star\}.
\]
If $a\notin B$, Lemma~\ref{lem:ralex-safe-one-block} with $s=s_j$ gives
\[
\Pr\!\left(Y_{j,a}=1\mid \widetilde{\mathcal G}_{j-1}\right)
\ge
\vartheta_1\pi(a)
=
\kappa(a)
\qquad\text{for every }j\in\{1,\dots,m_\star\}.
\]
Thus, for every arm $a\in A$ and every block $j$, we have the unified bound
\[
\Pr\!\left(Y_{j,a}=1\mid \widetilde{\mathcal G}_{j-1}\right)
\ge
\kappa(a).
\]
Applying Lemma~\ref{lem:adapted-chernoff} to the adapted Bernoulli sequence $\{Y_{j,a}\}_{j=1}^{m_\star}$ yields
\[
\Pr\!\left(
\sum_{j=1}^{m_\star}Y_{j,a}\le \frac{m_\star\kappa(a)}{2}
\right)
\le
\exp\!\left(-\frac{1-\log 2}{2}\,m_\star\kappa(a)\right).
\]
Because $\Delta(a)\le 1$ for every arm and because of the definition of $\Psi_{\mathrm{RA}}^{\mathrm{gain}}(\delta)$, choosing the universal constant $C$ large enough makes the right-hand side at most $\delta/(2n)$. A union bound over all arms gives an event $\widetilde{\mathcal E}_{\mathrm{vis}}^{\mathrm{gain}}$ of probability at least $1-\delta/2$ on which
\[
\sum_{j=1}^{m_\star}Y_{j,a}\ge \frac{m_\star\kappa(a)}{2}
\qquad\text{for every }a\in A.
\]
Since the endpoint count is dominated by the total number of visits,
\[
\widetilde\varphi_{t_\star}(a)\ge \sum_{j=1}^{m_\star}Y_{j,a}
\qquad\text{for every }a\in A,
\]
so on $\widetilde{\mathcal E}_{\mathrm{vis}}^{\mathrm{gain}}$,
\[
\widetilde\varphi_{t_\star}(a)\ge \frac{m_\star\kappa(a)}{2}
\qquad\text{for every }a\in A.
\]

We now prove that the perpetual stopping rule must hold at time $t_\star$ on the event
\[
\widetilde{\mathcal E}_{\mathrm{conf}}\cap \widetilde{\mathcal E}_{\mathrm{vis}}^{\mathrm{gain}}.
\]
Fix any suboptimal arm $a\neq a^\star$. The visit lower bound gives
\[
\widetilde w_{t_\star}(a)
\le
\sqrt{\frac{\log(4nT/\delta)}{m_\star\kappa(a)}}.
\]
Because
\[
m_\star\ge C\,\frac{\log(4nT/\delta)}{\kappa(a)\Delta(a)^2},
\]
choosing $C\ge 16$ yields
\[
\widetilde w_{t_\star}(a)\le \frac{\Delta(a)}{4}.
\]
For the optimal arm,
\[
\widetilde w_{t_\star}(a^\star)
\le
\sqrt{\frac{\log(4nT/\delta)}{m_\star\kappa(a^\star)}}
\le
\frac{\Delta_{\min}}{4}
\le
\frac{\Delta(a)}{4},
\]
again once $C\ge 16$.

On $\widetilde{\mathcal E}_{\mathrm{conf}}$, we therefore have
\[
\widetilde{\operatorname{LCB}}_{t_\star}(a^\star)
=
\widetilde{\hat\mu}_{t_\star}(a^\star)-\widetilde w_{t_\star}(a^\star)
\ge
\mu(a^\star)-2\widetilde w_{t_\star}(a^\star)
\]
and
\[
\widetilde{\operatorname{UCB}}_{t_\star}(a)
=
\widetilde{\hat\mu}_{t_\star}(a)+\widetilde w_{t_\star}(a)
\le
\mu(a)+2\widetilde w_{t_\star}(a).
\]
Subtracting gives
\[
\widetilde{\operatorname{LCB}}_{t_\star}(a^\star)-\widetilde{\operatorname{UCB}}_{t_\star}(a)
\ge
\Delta(a)-2\widetilde w_{t_\star}(a^\star)-2\widetilde w_{t_\star}(a)
\ge 0.
\]
Thus
\[
\widetilde{\operatorname{LCB}}_{t_\star}(a^\star)
\ge
\max_{a\neq a^\star}\widetilde{\operatorname{UCB}}_{t_\star}(a).
\]
Every confidence width is strictly positive at time $t_\star$, so for each $a\neq a^\star$,
\[
\widetilde{\hat\mu}_{t_\star}(a^\star)
\ge
\widetilde{\operatorname{LCB}}_{t_\star}(a^\star)
\ge
\widetilde{\operatorname{UCB}}_{t_\star}(a)
>
\widetilde{\hat\mu}_{t_\star}(a).
\]
Therefore $a^\star$ is the unique empirical leader at time $t_\star$, and the perpetual stopping rule must have triggered by time $t_\star$ with the correct output.

Finally,
\[
\Pr\!\left(
\widetilde{\mathcal E}_{\mathrm{conf}}\cap \widetilde{\mathcal E}_{\mathrm{vis}}^{\mathrm{gain}}
\right)
\ge 1-\delta.
\]
Hence
\[
\Pr\!\left(
\widetilde\sigma\le t_\star
\text{ and }
\widetilde b_{\widetilde\sigma}=a^\star
\right)
\ge 1-\delta.
\]
Since the actual RALEX process agrees with the perpetual one up to the actual stopping time, the same conclusion holds for the actual process:
\[
\Pr\!\left(
\sigma\le t_\star
\text{ and }
\hat a^\star=a^\star
\right)
\ge 1-\delta.
\]
This proves the theorem. The final bound
\[
R_{\mathrm{learn}}(T)\le T_{\mathrm{gain}}^{\mathrm{nom}}
\]
is immediate from $\Delta_{\max}\le 1$.
\end{proof}

\begin{proof}[Proof of Corollary~\ref{cor:ralex_gain}]
If $\mathcal H\subseteq B$ and
\[
\kappa_{\mathrm{gain}}(a)\ge (1+\zeta)\vartheta_1\pi(a)
\qquad\text{for every }a\in \mathcal H,
\]
then for each arm in the bottleneck set,
\[
\kappa(a)\ge (1+\zeta)\vartheta_1\pi(a).
\]
Since the bottleneck set contains the optimal arm and every suboptimal arm that maximizes the leading complexity term, substituting this bound into the definition of $\Psi_{\mathrm{RA}}^{\mathrm{gain}}(\delta)$ yields
\[
\Psi_{\mathrm{RA}}^{\mathrm{gain}}(\delta)
=
O\!\left(
\frac{\log(4nT/\delta)}{(1+\zeta)\vartheta_1}
\max\left\{
\frac{1}{\pi(a^\star)\Delta_{\min}^2},
\max_{a\neq a^\star}\frac{1}{\pi(a)\Delta(a)^2}
\right\}
\right).
\]
Multiplying by $\tau_1$ and adding the burn-in term $t_{\mathrm{gain}}$ gives the stated bound on $T_{\mathrm{gain}}^{\mathrm{nom}}$. The explicit improvement factor is exactly $(1+\zeta)^{-1}$.
\end{proof}

\end{refsection}

\end{document}